\documentclass[11pt]{article}

\usepackage[margin=1in]{geometry}
\usepackage{microtype}
\usepackage{graphicx}
\usepackage{subcaption}
\usepackage{booktabs}
\usepackage{hyperref}
\usepackage{xurl}
\usepackage{float}
\usepackage{enumitem}
\usepackage{authblk}
\usepackage{natbib}

\usepackage{amsmath}
\usepackage{amssymb}
\usepackage{mathtools}
\usepackage{amsthm}
\usepackage[capitalize,noabbrev]{cleveref}

\newtheorem{theorem}{Theorem}
\newtheorem{lemma}{Lemma}
\newtheorem{corollary}{Corollary}
\newtheorem{proposition}{Proposition}
\newtheorem{remark}{Remark}
\newtheorem{assumption}{Assumption}
\crefname{assumption}{Assumption}{Assumptions}

\begin{document}

\title{Coupling-robust accuracy in multiphysics physics-informed neural networks
  via Kronecker-preconditioned optimization}

\author[1]{Youngjae Park\textsuperscript{\,$\dagger$}}
\author[2]{Jaemin Kim\textsuperscript{\,$\dagger$}}
\author[1]{Junghwa Hong\thanks{Corresponding author: \texttt{hongjh32@korea.ac.kr}}}
\affil[1]{Dept.\ of Control and Instrumentation Engineering, Korea University, Sejong, South Korea}
\affil[2]{BK21 FOUR Smart Mobility Education and Research Team, Korea University, Sejong, South Korea}

\date{Preprint. August 2026.}

\maketitle
\let\thefootnote\relax\footnotetext{\textsuperscript{$\dagger$}\,These authors contributed equally.}
\setcounter{footnote}{0}

\begin{abstract}
Physics-informed neural networks (PINNs) offer a mesh-free route to solving coupled multiphysics systems, but their accuracy degrades systematically as inter-equation coupling strengthens, and inverse-gradient-norm loss balancing alone does not reliably prevent this failure. This study explains why coupling degrades PINN training and identifies an optimizer structure that removes the dependence, replacing case-by-case tuning with a principled remedy. Through a neural tangent kernel analysis, we prove that the standard kernel's spectral radius grows as $\Omega(\gamma^2)$ with coupling strength $\gamma$, whereas block-diagonal Gauss--Newton (GN) preconditioning bounds it by the number of networks $S$, independent of $\gamma$; no diagonal preconditioner recovers this bound for any coupling type or loss weighting. We realize block-diagonal GN preconditioning through the Kronecker-preconditioned optimizer SOAP combined with inverse-gradient-norm loss balancing (SOAP+GradNorm) and evaluate it across 222 experiments on four benchmarks of increasing difficulty. Across all systems, SOAP+GradNorm is the only configuration whose degradation remains bounded in every regime tested: it preserves weak-coupling accuracy in linear problems and limits degradation to $2.3\times$ in the nonlinear Nernst--Planck--Poisson system, whereas Adam-based training leaves the $L_2$ error above the $0.1$ failure threshold. The same effect applies to a six-residual, four-network, 2D electro-osmotic flow where the electric double layer is resolved down to a Debye length of $\varepsilon = 0.01$ on an $x$-invariant reference solution. These results recast coupling-induced accuracy loss as a problem of the preconditioner's structure rather than loss weighting and identify Kronecker preconditioning as a structural lever for training PINNs on strongly coupled, stiff multiphysics systems.
\end{abstract}

\section{Introduction}

Physics-informed neural networks (PINNs) \citep{raissi2019physics} embed partial differential equation (PDE) residuals into the loss function, enabling mesh-free solutions. These networks have achieved broad success \citep{karniadakis2021physics,hao2024pinnacle} across the fields of fluid mechanics \citep{raissi2020hidden,jin2021nsfnets}, heat transfer \citep{cai2021physics}, bioengineering \citep{kissas2020machine,yin2021non}, materials science \citep{zhang2022analyses}, and subsurface transport \citep{he2020physics}; see \citet{cuomo2022scientific} for a review.

Alongside these achievements, a line of work established that the difficulty associated with training PINNs is fundamentally an optimization and conditioning problem. \citet{krishnapriyan2021characterizingpossiblefailuremodes} showed that PINNs can fail even on simple problems because of ill-conditioned loss landscapes, and \citet{wang2022and} traced this failure to the spectral structure of the neural tangent kernel (NTK), where disparate eigenvalues cause loss components to converge at incompatible rates. This diagnosis motivated a first generation of remedies but was developed almost entirely for single-equation problems.

Two complementary branches followed. One improved the representation or training dynamics of a single network---Fourier-feature embeddings to counter spectral bias \citep{wang2021eigenvector}, adaptive activations \citep{jagtap2020adaptive}, domain decomposition \citep{jagtap2020extended}, residual-based adaptive sampling \citep{wu2023comprehensive}, and causal or sequential training schedules \citep{wang2022respecting,mattey2022novel}. The other rebalanced competing loss terms, reweighting residuals from gradient statistics \citep{wang2021understanding} or adaptive attention \citep{mcclenny2023self,anagnostopoulos2024residual}. Although these advances substantially improve single-PDE accuracy, loss balancing remains a heuristic applied after the fact. Furthermore, none of these methods targets the imbalance that arises specifically from inter-equation coupling.

For coupled systems where multiple PDEs interact through shared variables and residuals compete for optimizer capacity, existing PINN approaches manage the interaction architecturally, including mixed variational formulations \citep{harandi2024mixed}, sequential field-by-field (stress-split) training \citep{haghighat2022physics}, and segregated per-field networks \citep{sun2024physics}. These choices are effective but are adopted heuristically; the coupling itself is treated as an implementation obstacle, not an object of analysis.

Consequently, despite this progression from diagnosis to single-equation remedies to multiphysics architectures, to our knowledge, no prior study has systematically quantified how accuracy scales with coupling strength across multiple systems, nor explained why preconditioning---rather than loss reweighting---resolves the resulting degradation.

In this study, we aim to (i) characterize how PINN accuracy degrades as inter-equation coupling strengthens across representative multiphysics benchmarks and (ii) identify and prove why a specific optimizer structure---block-diagonal Gauss--Newton (GN) preconditioning---removes this dependence, thereby replacing case-by-case loss-weight tuning with a principled remedy. We then realize this structure in practice and evaluate it at scale.

\textbf{Contributions.}\;
(1)~We prove that block-diagonal GN preconditioning bounds $\lambda_{\max}(K_P) \leq S$, independently of the coupling strength $\gamma$, while the unpreconditioned NTK grows as $\Omega(\gamma^2)$. We verify the growth on the linearly coupled benchmarks and the equality $\lambda_{\max}(K_P) = S$ on all three 1D benchmarks.
(2)~We show that Adam's diagonal preconditioning cannot recover this bound in the overparameterized regime for any coupling type or loss weighting and that its residual-dynamics kernel grows as $\Theta(\gamma)$, concentrating spectral energy away from coupled residuals. For one-way coupling, we prove that this limitation is class-wide in iteration complexity: under linearized residual dynamics, every diagonal preconditioner, adaptive or not, requires $\Omega(\gamma)$ iterations on the driving residual ($\Omega(\gamma^2)$ if fixed), where block GN preconditioning requires $O(1)$.
(3)~Across 222 experiments on four benchmark systems, SOAP+GradNorm is the only configuration whose degradation remains bounded in every regime tested, and the only one that succeeds on the 2D six-residual system, preserving weak-coupling accuracy in the linear systems and limiting degradation to $2.3\times$ in the nonlinear Nernst--Planck--Poisson system.
(4)~We show that the same optimizer-structure effect persists when the system is scaled up to a 2D, six-residual electro-osmotic problem solved with four networks, in which the Debye layer is resolved rather than replaced (dimensionless Debye length $\varepsilon = 0.01$).

\section{Background}
\label{sec:background}

\subsection{Multiphysics PINNs}
\label{sec:multiphysics}

Consider $N$ coupled PDEs $\mathcal{R}_i[\hat{\mathbf{u}}](x) = 0$,
$i = 1,\ldots,N$, with boundary conditions
$\mathcal{B}_j[\hat{\mathbf{u}}](x) = 0$ on $\partial\Omega$,
where $\hat{\mathbf{u}} = (\hat{u}_1, \ldots, \hat{u}_N)$ and residuals are evaluated via automatic differentiation \citep{baydin2018automatic}.
Two network architectures are commonly used.
In the \emph{single-network} approach, one network $f(x;\,\theta) : \mathbb{R}^d \to \mathbb{R}^N$ outputs all the fields simultaneously, and all parameters are shared across fields \citep{raissi2019physics,lu2021deepxde}.
In the \emph{segregated-network} approach, each field is approximated by an independent network $\hat{u}_i(x;\,\theta_i)$ with disjoint parameters ($\theta_i \cap \theta_j = \emptyset$ for $i \neq j$), as adopted in electrokinetic multiphysics \citep{merdasi2023physics,sun2024physics} and multispecies heat transfer \citep{laubscher2021simulation}.
This design has been adopted in several recent multiphysics PINN studies and enables per-field architectural tuning.
In a direct comparison on a coupled multispecies problem, the segregated variant was reported to reach substantially lower losses and maintain species conservation more effectively than a single shared network \citep{laubscher2021simulation}.
Various enhancements to segregated training have been proposed, including sequential field-by-field optimization \citep{haghighat2022physics} and mixed variational formulations \citep{harandi2024mixed}.
The theoretical analysis in \cref{sec:theory} is stated for $S \geq 1$ parameter groups and therefore covers both architectures, with the single network corresponding to $S = 1$ (see \cref{rem:single_net}).

The training loss is
\begin{equation}
  \mathcal{L}(\theta)
  = \sum_{i} \lambda_i \frac{1}{N_r}\sum_{k=1}^{N_r}\mathcal{R}_i(x_k^r)^2
  + \sum_{j} \lambda_j^{\mathrm{b}}
    \frac{1}{N_b}\sum_{k=1}^{N_b}\mathcal{B}_j(x_k^b)^2,
  \label{eq:loss}
\end{equation}
where $\{x_k^r\}$ and $\{x_k^b\}$ are interior collocation and
boundary points, respectively, and $\lambda_i$ and $\lambda_j^{\mathrm{b}}$
are loss weights determined by the balancing scheme
(fixed at one when no balancing is applied; see \cref{sec:balancing}).
Throughout, we collect the residual loss weights into the block-diagonal matrix $\Lambda = \operatorname{blkdiag}(\lambda_1 I_{N_r}, \ldots, \lambda_M I_{N_r})$; we write $\Lambda > 0$ when $\lambda_i > 0$ for every $i$, and call $\Lambda$ \emph{uniform} when $\lambda_1 = \cdots = \lambda_M$.

\subsection{SOAP Optimizer}

SOAP \citep{vyas2025soap} applies Kronecker-factored preconditioning to each weight matrix, building on the Shampoo algorithm \citep{gupta2018shampoo}: a layer-wise preconditioner is constructed from running averages of left and right gradient outer products, replacing the standard gradient update $\theta \leftarrow \theta - \eta \nabla\mathcal{L}$ with $\theta \leftarrow \theta - \eta\, P\, \nabla\mathcal{L}$. \citet{morwani2024new} showed that SOAP's Kronecker factors approximate the optimal Kronecker factorization of each layer's GN matrix $H_{\mathrm{GN}}^{(l)}$, and \citet{wang2026gradient} further established that under a settled-phase assumption on Adam's second-moment exponent ($s = 2$), SOAP's update approximates the block-diagonal GN preconditioner. \citet{wang2026gradient} also demonstrated state-of-the-art single-PDE PINN accuracy across 10 benchmarks. However, these benchmarks are evaluated at a fixed set of PDE parameters. How accuracy scales as inter-equation coupling strengthens is not addressed there, and that is the question examined in this study. SOAP's Kronecker preconditioning operates per weight matrix (layer-wise). In our segregated-network setting, where each network maintains disjoint parameters, there are no cross-network weight matrices. Therefore, SOAP's layer-wise block-diagonal structure is automatically network-wise block-diagonal---a property we formalize in \cref{prop:soap_net} and associate with the GN preconditioning analyzed in \cref{sec:theory}. This gap motivates the present study.

\subsection{Gradient Balancing}
\label{sec:balancing}
When no balancing is applied, all loss weights are fixed at $\lambda_i = 1$.

\textbf{Inverse gradient-norm weighting (GradNorm).}\footnote{Following the terminology of \citet{wang2026gradient}. Not to be confused with the GradNorm algorithm of \citet{chen2018gradnorm}, which balances tasks by matching relative training rates rather than equalizing gradient norms.}\; Following \citet{wang2026gradient}, we set $\lambda_i \propto 1/\|\nabla_\theta \mathcal{L}_i\|_2$ with an exponential moving average (momentum~$= 0.9$, update frequency $= 1000$ steps). This equalizes $\|\lambda_i \nabla \mathcal{L}_i\|_2$ across all $i$, a principle introduced in \citep{wang2021understanding} and systematized in \citep{wang2023expert}. We adopted this as the primary balancing scheme because the equal gradient-norm condition is required for the single-step convergence bound of preconditioned gradient descent to be tight \citep{wang2026gradient}.

\section{Benchmark Systems}
\label{sec:benchmarks}

In this study, we use four systems: one-dimensional thermoelasticity, a three-species reaction--diffusion chain, the one-dimensional Nernst--Planck--Poisson system, and a 2D electro-osmotic flow (EOF).

\subsection{1D Thermoelasticity (Thermo: Linear, One-way Coupling)}
\label{sec:thermo}

We use the one-dimensional steady thermoelasticity system. The following equations
\begin{equation}
  -\kappa\, T'' = f_T, \qquad -E\, u'' + \gamma\, T' = f_u
\end{equation}
are valid for $x \in [0, 1]$ with Dirichlet boundary conditions (BCs) set according to the exact solutions,
\begin{equation}
  T(0) = T(1) = 1, \qquad u(0) = u(1) = 0.
\end{equation}
The coupling is one-way: temperature $T$ drives displacement $u$ through the thermal stress term $\gamma T'$ but $u$ does not affect $T$. Following the method of manufactured solutions, we prescribe $T^*(x) = \sin(\pi x) + 1$, $u^*(x) = \sin(2\pi x)$, which determine the source terms:
\begin{equation}
  f_T = \kappa\pi^2\sin(\pi x), \qquad f_u = E(2\pi)^2\sin(2\pi x) + \gamma\pi\cos(\pi x).
\end{equation}
The coupling parameter $\gamma \in \{1, 5, 10, 25, 50, 100\}$, where the thermal conductivity and the elastic modulus are $\kappa = E = 1$.

\subsection{1D Reaction--Diffusion [(RD): Linear, Bidirectional]}

We use a three-species reaction--diffusion chain $A \rightleftharpoons B \rightleftharpoons C$. The equations
\begin{align}
  D_1\, c_A'' - k_1 c_A + k_2 c_B &= f_1 \notag\\
  D_2\, c_B'' + k_1 c_A - (k_2 + k_3) c_B + k_4 c_C &= f_2 \\
  D_3\, c_C'' + k_3 c_B - k_4 c_C &= f_3 \notag
\end{align}
are valid for $x \in [0, 1]$ with Dirichlet BCs set according to the exact solutions:
\begin{gather}
  c_A(0) = c_A(1) = 2, \quad c_B(0) = 3,\; c_B(1) = 1, \notag\\
  c_C(0) = c_C(1) = 2.
\end{gather}
We denote the forward rate constant $k_{\mathrm{f}} \equiv k_1 = k_3$. Backward rates are $k_2 = k_4 = 1$, with $k_{\mathrm{f}} \in \{1, 5, 10, 30, 50, 100\}$ (abbreviated $k$ in tables and figures), and the species diffusivities $D_i = 1$. Manufactured solutions $c_A^* = \sin(\pi x) + 2$, $c_B^* = \cos(\pi x) + 2$, $c_C^* = \sin(2\pi x) + 2$ determine the following source terms
\begin{align}
  f_1 &= -\pi^2\sin(\pi x) - k_{\mathrm{f}}\bigl(\sin(\pi x) + 2\bigr) + \cos(\pi x) + 2, \notag\\
  f_2 &= -\pi^2\cos(\pi x) + k_{\mathrm{f}}\bigl(\sin(\pi x) + 2\bigr) \notag\\
      &\quad - (1+k_{\mathrm{f}})\bigl(\cos(\pi x) + 2\bigr) + \sin(2\pi x) + 2, \notag\\
  f_3 &= -(2\pi)^2\sin(2\pi x) + k_{\mathrm{f}}\bigl(\cos(\pi x) + 2\bigr) - \bigl(\sin(2\pi x) + 2\bigr).
\end{align}

\subsection{1D Nernst--Planck--Poisson [(NP+P): Nonlinear]}

We use the 1D steady NP+P system. The equations
\begin{align}
  c_+'' + (c_+ \phi')' &= 0 \notag\\
  c_-'' - (c_- \phi')' &= 0 \\
  \varepsilon^2 \phi'' + (c_+ - c_-) &= 0 \notag
\end{align}
are valid on $x \in [0, 1]$ with Boltzmann equilibrium BCs ($\zeta = 1.0$):
\begin{equation}
  \phi(0) = \zeta,\; c_\pm(0) = e^{\mp\zeta}, \qquad \phi(1) = 0,\; c_\pm(1) = 1.
\end{equation}
The exact solution is the 1D Poisson--Boltzmann equilibrium, obtained by solving the nonlinear boundary value problem (BVP) $\varepsilon^2\phi'' = 2\sinh(\phi)$ numerically (\texttt{solve\_bvp} from SciPy \citep{virtanen2020scipy}), with $c_\pm^* = e^{\mp\phi^*}$. The parameter $\varepsilon \in \{1.0, 0.5, 0.2, 0.1\}$ controls coupling stiffness.

\subsection{2D EOF}
\label{sec:eof}

EOF arises from the coupling of ionic transport, electrostatics, and viscous flow within the electric double layer (EDL) that forms near charged surfaces \citep{probstein2005physicochemical,kirby2004zeta,mani2020electroconvection}.

Prior PINN studies on electrokinetic transport simplify the coupled system differently: a prescribed space-charge density without ion transport \citep{tao2025lstm}, a Debye-length nondimensionalization that removes the stiffness parameter \citep{huang2025enriched}, an equilibrium Poisson--Boltzmann closure linearized for $\zeta \ll 25$\,mV \citep{merdasi2023physics}, a Helmholtz--Smoluchowski slip condition in the flow-coupled case \citep{sun2024physics}, and operator learning from spectral-element data instead of residual minimization \citep{cai2021deepm}. In contrast, this benchmark assembles the difficulties of solving a coupled problem by residual minimization alone: six residuals competing for optimizer capacity, two-way coupling between charge transport and flow, nonlinear migration terms, a wall layer whose width is set by a stiffness parameter, and field magnitudes that differ by up to four orders.

The fully coupled NP+P+Stokes system on a rectangular channel $\Omega = [0,1]^2$ can be defined as
\begin{align}
  \varepsilon^2 \nabla^2 \phi + (c_+ - c_-) &= 0 \notag\\
  \nabla^2 c_+ + \nabla \cdot (c_+ \nabla\phi) &= 0 \notag\\
  \nabla^2 c_- - \nabla \cdot (c_- \nabla\phi) &= 0 \\
  -\nabla p + \mu \nabla^2 \mathbf{u} + E_x(c_+ - c_-)\hat{x} &= 0 \notag\\
  \nabla \cdot \mathbf{u} &= 0, \notag
\end{align}
with $E_x = \mu = 1$, denoting the applied axial electric field and the dynamic viscosity, respectively. The BCs are defined as
\begin{alignat}{2}
  &\text{Walls } (y = 0,\, 1)\text{:} \notag\\
  &\quad \phi = \zeta,\;\; c_\pm = e^{\mp\zeta},\;\; u = v = 0, \notag\\
  &\text{Channel ends } (x = 0,\, 1)\text{:} \notag\\
  &\quad \phi = \phi_{\mathrm{ref}}(y),\;\; c_\pm = c_{\pm,\mathrm{ref}}(y), \notag\\
  &\quad u = u_{\mathrm{ref}}(y),\;\; v = 0,\;\; p = 0.
  \label{eq:2d_bc}
\end{alignat}
Here, $\phi_{\mathrm{ref}}(y)$, $c_{\pm,\mathrm{ref}}(y)$, and $u_{\mathrm{ref}}(y) = E_x \varepsilon^2 (\phi_{\mathrm{ref}}(y) - \zeta)/\mu$ are the fully developed 1D Poisson--Boltzmann equilibrium solution (obtained via \texttt{solve\_bvp}), used as an $x$-independent exact solution on the 2D domain. The $x$-invariance is a deliberate benchmark choice; it provides an exact reference for the $L_2$ metric without recourse to a numerical solver and isolates coupling stiffness from geometric complexity; the evaluation of genuinely 2D configurations (axially varying $\zeta$, applied pressure gradients) constitutes future work. The parameters were $\zeta = 1.0$ and $\varepsilon \in \{0.2, 0.1, 0.05\}$, extended to $\varepsilon = 0.01$ in the architecture comparison of \cref{sec:2d_results} (\cref{tab:2d_npps}). This is discussed in \cref{sec:discussion}.

\section{Theoretical Analysis}
\label{sec:theory}

We now analyze why coupling strength degrades gradient descent (GD) (but not GN preconditioning) by considering the NTK \citep{jacot2018neural}. While NTK-based convergence analysis for single-equation PINNs has been developed by \citet{wang2022and}, the multiphysics setting introduces a qualitatively different spectral structure that we characterize here.

\subsection{Setup and Notation}

Consider $S$ independent networks with disjoint parameters $\theta = (\theta_1, \ldots, \theta_S)$ and $M$ PDE residuals evaluated at $N_r$ collocation points. The residual Jacobian has the following block structure:
\begin{equation}
  J = \bigl[J_{i,\theta_k}\bigr]_{\substack{i=1,\ldots,M \\ k=1,\ldots,S}}
  = \begin{bmatrix}
    J_{1,\theta_1} & J_{1,\theta_2} & \cdots & J_{1,\theta_S} \\
    J_{2,\theta_1} & J_{2,\theta_2} & \cdots & J_{2,\theta_S} \\
    \vdots & \vdots & \ddots & \vdots \\
    J_{M,\theta_1} & J_{M,\theta_2} & \cdots & J_{M,\theta_S}
  \end{bmatrix}
  \in \mathbb{R}^{MN_r \times p},
  \label{eq:jacobian}
\end{equation}
where each row-block group corresponds to a PDE residual and each column-block group to a network: $J_{i,\theta_k} = \partial \mathbf{R}_i / \partial \theta_k \in \mathbb{R}^{N_r \times p_k}$, $\mathbf{R}_i = (\mathcal{R}_i(x_1), \ldots, \mathcal{R}_i(x_{N_r}))^\top$, and $p = \sum_k p_k$.

\begin{assumption}[Segregated networks]
\label{asm:segregated}
The $S$ networks have disjoint parameter sets: $\theta_i \cap \theta_j = \emptyset$ for all $i \neq j$.
\end{assumption}

\begin{assumption}[Linear coupling]
\label{asm:linear_coupling}
There exist indices $(i, k)$, $i \neq k$, such that $J_{i,\theta_k} = \gamma\, C_{ik}$ with $C_{ik} \neq 0$ independent of~$\gamma$.
\end{assumption}

The unpreconditioned NTK is
\begin{equation}
  K = J J^\top \in \mathbb{R}^{MN_r \times MN_r}.
  \label{eq:ntk}
\end{equation}

The block-diagonal GN Hessian groups according to the network are
\begin{equation}
  H = \operatorname{blkdiag}(H_1, \ldots, H_S), \quad H_k = J_{\theta_k}^\top J_{\theta_k},
  \label{eq:hessian}
\end{equation}
where $J_{\theta_k} \in \mathbb{R}^{MN_r \times p_k}$ collects all residual derivatives with respect to $\theta_k$.
Note that $H$ retains only the diagonal blocks of the full GN Hessian $J^\top J$. The off-diagonal blocks $J_{\theta_k}^\top J_{\theta_l}$ ($k \neq l$), which encode cross-network curvature interactions, are discarded.
The preconditioned NTK is
\begin{equation}
  K_P = J\, H^{+}\, J^\top,
  \label{eq:kp}
\end{equation}
where $H^{+}$ denotes the Moore--Penrose pseudoinverse.

More generally, if $P$ denotes the preconditioner, the linearized residual dynamics $\mathbf{r}_{t+1} = (I - \eta J P J^\top)\mathbf{r}_t$ converge when $\eta < 2/\lambda_{\max}(J P J^\top)$ \citep{wang2022and}. Different preconditioners yield different maximum stable learning rates.

\begin{table}[ht]
\caption{Preconditioner hierarchy for segregated-network physics-informed neural networks (PINNs) with $S$ networks, each having $L$ layers and $p_k$ parameters. The projector decomposition (\cref{lem:projector}) ensures that $\lambda_{\max}(K_P) \leq S$ for block-diagonal Gauss--Newton (GN). Adam's diagonal scaling destroys this structure (\cref{prop:adam}).}
\label{tab:preconditioners}
\centering
\resizebox{\textwidth}{!}{%
\begin{tabular}{@{}lllll@{}}
\toprule
Preconditioner & Definition & $\lambda_{\max}$ & Cost/step & Reference \\
\midrule
None (gradient descent) & $P = I$ & $\Omega(\gamma^2)$ & $O(p)$ & \cref{thm:gd} \\
Adam (transient) & $P \propto D^{-1/2}$,\; $D = \operatorname{diag}(J^\top J)$ & $\Theta(\gamma)$ & $O(p)$ & \cref{prop:adam_eta}; \cref{rem:adam_eta} \\
Adam (settled) & $P \propto D^{-1}$,\; $D = \operatorname{diag}(J^\top J)$ & $\geq p/(MN_r) \gg S$ & $O(p)$ & \cref{prop:adam} \\
Block-diag GN & $P = H^{+}$ (\cref{eq:hessian}) & $\leq S$ & $O(\sum_k p_k^3)$ & \cref{thm:gn} \\
Shampoo with Adam in the preconditioner's eigenbasis (SOAP) & $P \approx \operatorname{blkdiag}(L_k^{(l)} \otimes R_k^{(l)})^{-1/2}$ & $\leq SL$ & $O(\sum_l p_l^{3/2})$ & \cref{prop:soap_net} \\
\bottomrule
\end{tabular}}
\end{table}

Two families of bounds emerge. Without block-diagonal preconditioning, the spectral radius grows as $\Omega(\gamma^2)$ (GD), as $\Theta(\gamma)$ (Adam, transient), and is at least $p/(MN_r)$ (Adam, settled), which already exceeds $S$. On the other hand, block-diagonal GN attains $\leq S$, and its layer-wise form attains $\leq SL$ under \cref{asm:soap_approx}. Both are independent of $\gamma$, and the conditions under which each result holds are summarized in \cref{tab:scope}.

The key distinction in \cref{tab:preconditioners} is between diagonal and block preconditioners.
In coupled systems, the Jacobian~\eqref{eq:jacobian} contains off-diagonal blocks $J_{i,\theta_k} = \gamma\, C_{ik}$ ($i \neq k$) that contribute $O(\gamma^2)$ terms to the NTK, inflating $\lambda_{\max}(K)$.
Block-diagonal GN inverts each per-network GN Hessian block $H_k$ in full, rather than only its diagonal entries, yielding a preconditioned NTK that decomposes into orthogonal projectors with $\lambda_{\max}(K_P) \leq S$, independent of~$\gamma$.
Adam's diagonal extraction discards the off-diagonal entries of each $H_k$, destroying this projector structure.
The following subsections formalize these bounds; specifically, \cref{sec:ntk_grows} proves the $\Omega(\gamma^2)$ lower bound, \cref{sec:preconditioned_ntk} proves the coupling-independent upper bound $\lambda_{\max}(K_P) \leq S$, and \cref{sec:adam} establishes why Adam fails to provide the same protection.

\subsection{Standard NTK Grows with Coupling}
\label{sec:ntk_grows}

Consider a general linearly coupled PDE system where each residual takes the form
\begin{equation}
  \mathcal{R}_i = \mathcal{L}_i[\hat{u}_i] + \gamma\, \mathcal{G}_i[\hat{u}_k] - f_i = 0,
  \label{eq:coupled_pde}
\end{equation}
where $\mathcal{L}_i$ is the principal differential operator acting on network~$i$'s own output, $\mathcal{G}_i$ is a coupling operator acting on network~$k$'s output ($k \neq i$), and $\gamma$ is the coupling parameter. \Cref{asm:linear_coupling} is satisfied whenever each residual takes this form as differentiating with respect to $\theta_k$ yields $J_{i,\theta_k} = \gamma\, C_{ik}$, where $C_{ik} = \partial \mathcal{G}_i[\hat{u}_k]/\partial \theta_k$ is independent of~$\gamma$.

Both linearly coupled benchmarks have this structure: the thermal stress term $\gamma T'$ in thermoelasticity produces $J_{u,\theta_T} = \gamma\,\partial T'/\partial\theta_T$, and the forward reaction term $k_{\mathrm{f}}\,c_A$ coupling into $R_B$ in reaction--diffusion produces $J_{B,\theta_A} = k_{\mathrm{f}}\,\partial c_A/\partial\theta_A$.

\begin{theorem}[GD spectral bound for linearly coupled systems]
\label{thm:gd}
Under \cref{asm:segregated,asm:linear_coupling}, the standard NTK satisfies
\begin{equation}
  \lambda_{\max}(K) \geq \gamma^2 \,\sigma_{\max}^2(C_{ik}),
  \label{eq:gd_bound}
\end{equation}
where $\sigma_{\max}(C_{ik})$ is the largest singular value of~$C_{ik}$. As $\sigma_{\max}(C_{ik})$ depends only on the network architecture and initialization (not~$\gamma$), the maximum stable learning rate satisfies $\eta_{\max}^{\mathrm{GD}} = O(1/\gamma^2)$.
\end{theorem}

\begin{proof}[Proof sketch]
The $(i,i)$-th diagonal block of $K = JJ^\top$ satisfies $K_{[i,i]} = \sum_{j=1}^{S} J_{i,\theta_j} J_{i,\theta_j}^\top \succeq J_{i,\theta_k} J_{i,\theta_k}^\top = \gamma^2 C_{ik} C_{ik}^\top$. As $K_{[i,i]}$ is a principal submatrix of the positive semidefinite matrix~$K$, $\lambda_{\max}(K) \geq \lambda_{\max}(K_{[i,i]}) \geq \gamma^2 \sigma_{\max}^2(C_{ik})$.
\end{proof}

\begin{remark}[Scope of \cref{thm:gd}]
\label{rem:gd_scope}
\Cref{thm:gd} requires $C_{ik}$ to be independent of~$\gamma$, which holds for the linearly coupled benchmarks (thermoelasticity and RD). It does not hold for NP+P, where the migration term $(c_\pm \phi')'$ makes the coupling block depend on the fields themselves; therefore, no $\gamma$-independent $C_{ik}$ exists and the bound does not apply.
\end{remark}

\subsection{Preconditioned NTK is Coupling-independent}
\label{sec:preconditioned_ntk}

Recall from \cref{eq:kp} that $K_P = J\,H^{+}\,J^\top$ is the preconditioned NTK under block-diagonal GN. The bound follows from a decomposition of $K_P$ into orthogonal projectors.

\begin{lemma}[Projector decomposition]
\label{lem:projector}
Under \cref{asm:segregated}, the preconditioned NTK admits the decomposition
\begin{equation}
  K_P = \sum_{k=1}^{S} P_k, \qquad P_k = J_{\theta_k}\,(J_{\theta_k}^\top J_{\theta_k})^{+}\, J_{\theta_k}^\top,
  \label{eq:projector}
\end{equation}
where each $P_k$ is the orthogonal projector onto $\mathrm{col}(J_{\theta_k})$.
\end{lemma}

Because each $P_k$ is an orthogonal projector, the Rayleigh quotient of their sum is bounded by $S$.

\begin{theorem}[GN spectral bound]
\label{thm:gn}
For any segregated-network PINN with $S$ independent networks and block-diagonal GN preconditioning, $\lambda_{\max}(K_P) \leq S$.
This bound is independent of the coupling parameter $\gamma$, PDE structure, and the network parameterization. The only structural assumptions are disjoint parameters ($\theta_i \cap \theta_j = \emptyset$) and block-diagonal GN.
\end{theorem}

\begin{proof}
From \cref{lem:projector}, $K_P = \sum_{k=1}^{S} P_k$. For any unit vector $\mathbf{v}$,
\begin{equation}
  \mathbf{v}^\top K_P\, \mathbf{v} = \sum_{k=1}^{S} \|P_k \mathbf{v}\|^2 \leq \sum_{k=1}^{S} 1 = S,
\end{equation}
where $\|P_k \mathbf{v}\| \leq \|\mathbf{v}\| = 1$ as $P_k$ is an orthogonal projector.
\end{proof}

\begin{remark}[Scope and extensions]
\label{rem:scope}\label{rem:single_net}
The spectral bounds are algebraic identities at each $\theta$, not asymptotic limits. Only the $\gamma$-scaling is meaningful across kernels. All projector-based results (\cref{lem:projector,thm:gn,prop:soap_net,prop:adam}) extend to $S = 1$. The GD bound requires a crossover argument (\cref{app:single_net}).
\end{remark}

\begin{assumption}[SOAP as approximate layer-wise GN]
\label{asm:soap_approx}
SOAP's Kronecker factors $L_t \otimes R_t$ approximate the layer-wise GN matrix $H_{\mathrm{GN}}^{(l)} = J_{\theta_k^{(l)}}^\top J_{\theta_k^{(l)}}$ in the sense that each layer's preconditioner structurally corresponds to a layer-wise GN inverse.
\end{assumption}

This is motivated by the optimal Kronecker approximation result of \citet{morwani2024new}, which shows that $(L_t \otimes R_t)^{1/2}$ (up to trace normalization) is the optimal Kronecker approximation of $H_{\mathrm{GN}}^{(l)}$ so that SOAP's $(L_t \otimes R_t)^{-1/2}$ scaling corresponds to a per-layer GN inverse. \citet{wang2026gradient} combine this correspondence with the further assumption that the GN matrix approximates the true Hessian to connect SOAP with Newton's method. We use the correspondence alone only to derive the spectral bound of \cref{prop:soap_net}.

\begin{proposition}[SOAP spectral bound]
\label{prop:soap_net}
Under \cref{asm:segregated}, (1) SOAP's layer-wise preconditioner $H^{\mathrm{SOAP}} = \operatorname{blkdiag}(\tilde{H}_k^{(l)})$ is block-diagonal across networks. (2) When \cref{asm:soap_approx} holds exactly ($\tilde{H}_k^{(l)} = J_{\theta_k^{(l)}}^\top J_{\theta_k^{(l)}}$), each block yields an orthogonal projector $P_k^{(l)}$ onto $\mathrm{col}(J_{\theta_k^{(l)}})$ so that $K_P^{\mathrm{SOAP}} = \sum_{k,l} P_k^{(l)}$ and $\lambda_{\max}(K_P^{\mathrm{SOAP}}) \leq SL$, independent of coupling strength.
\end{proposition}

\begin{proof}
The same argument applies as that used for \cref{thm:gn}: $\mathbf{v}^\top K_P^{\mathrm{SOAP}}\, \mathbf{v} = \sum_{k,l} \|P_k^{(l)}\mathbf{v}\|^2 \leq SL$ for any unit $\mathbf{v}$.
\end{proof}

\subsection{Why Adam Does Not Resolve Coupling}
\label{sec:adam}

Block-diagonal GN's projector decomposition (\cref{lem:projector}) in conjunction with the equality case of \cref{thm:gn} (\cref{app:proofs}) yields $\lambda_{\max}(K_P) = S$ in the overparameterized regime. We now show that Adam's diagonal preconditioning cannot recover this bound. The three statements below cover complementary regimes: (i) diagonal extraction destroys the projector structure for any coupling type (\cref{prop:adam}); (ii) the surviving kernel retains $\gamma$-dependent growth under linear coupling (\cref{prop:adam_eta}); and (iii) for one-way coupling, the limitation extends to every positive diagonal preconditioner without modeling Adam's second moment (\cref{prop:diag_lb}).

Adam's update divides the gradient elementwise by $\sqrt{\hat v_t}$, where $\hat v_t$ denotes Adam's second-moment estimate. In the transient phase, the standard diagonal-Fisher characterization $\hat v_t \approx \operatorname{diag}(J^\top J) =: D$ \citep{kingma2014adam} yields an effective preconditioner $P^{\mathrm{Adam}} \propto D^{-1/2}$. Near a minimum, $\hat v_t$ scales as $D^2$, yielding $P^{\mathrm{Adam}} \propto D^{-1}$, the settled-phase characterization of \citet{wang2026gradient}. Writing $c_j$ for the $j$-th column of $J$ and $\tilde{c}_j = c_j / \|c_j\|$ (so that $D_{jj} = \|c_j\|^2$), the two regimes induce two kernels. The residual-dynamics kernel
\begin{equation}
  K_{1/2}^{\mathrm{Adam}} := J D^{-1/2} J^\top = \sum_{j:\,D_{jj}>0} \|c_j\|\, \tilde{c}_j \tilde{c}_j^\top
  \label{eq:adam_half_kernel}
\end{equation}
governs the linearized residual iteration $\mathbf{r}_{t+1} \approx (I - \eta\, K_{1/2}^{\mathrm{Adam}})\,\mathbf{r}_t$ obtained by substituting $\Delta\theta = -\eta\, D^{-1/2} \nabla_\theta \mathcal{L}$ into $\Delta\mathbf{r} \approx J\, \Delta\theta$. Its settled-phase counterpart
\begin{equation}
  K_P^{\mathrm{Adam}} := J D^{-1} J^\top = \sum_{j:\,D_{jj}>0} \tilde{c}_j \tilde{c}_j^\top
  \label{eq:adam_norm_kernel}
\end{equation}
strips the column weights, isolating the directional structure. It is scale-invariant, which makes it the natural object for structural statements. The two kernels share identical rank-one directions and differ only in their weights ($\|c_j\|$ versus $1$). Momentum, bias correction, and Adam's $\epsilon$-offset are omitted from the kernel definitions in this subsection. For the linearized iteration, they rescale the stability threshold by $\gamma$-independent factors and leave the $\gamma$-scaling unchanged.

\begin{proposition}[Adam's preconditioned NTK lacks projector decomposition]
\label{prop:adam}
Under \cref{asm:segregated}, consider the Adam-preconditioned NTK $K_P^{\mathrm{Adam}} = J D^{-1} J^\top$ with $D = \operatorname{diag}(J^\top J)$. In the overparameterized regime ($p_k > MN_r$ for every network~$k$), no summand $Q_k = J_{\theta_k} D_k^{-1} J_{\theta_k}^\top$ is an orthogonal projector, and $\lambda_{\max}(K_P^{\mathrm{Adam}}) \geq p/MN_r > S$; therefore, the projector decomposition of \cref{lem:projector} and the bound of \cref{thm:gn} do not extend to $K_P^{\mathrm{Adam}}$.
\end{proposition}

\begin{proof}
As $\theta_i \cap \theta_j = \emptyset$, $D$ inherits the network partition:
\begin{equation}
  D = \operatorname{blkdiag}(D_1, \ldots, D_S), \qquad
  D_k = \operatorname{diag}(H_k),\quad H_k = J_{\theta_k}^\top J_{\theta_k},
\end{equation}
and consequently
\begin{equation}
  K_P^{\mathrm{Adam}} = \sum_{k=1}^{S} Q_k, \qquad
  Q_k = J_{\theta_k}\, D_k^{-1}\, J_{\theta_k}^\top .
\end{equation}
All the sums are over columns with $D_{jj} > 0$ (\cref{rem:bc_rows}).

For Adam, the comparison rests on a trace--rank argument. Each retained diagonal entry satisfies $[H_k]_{jj}/[D_k]_{jj} = 1$. Therefore,
\begin{equation}
  \operatorname{trace}(Q_k) = \operatorname{trace}(D_k^{-1} H_k) = p_k,
  \qquad \text{while} \quad
  \operatorname{rank}(Q_k) \leq MN_r.
\end{equation}
Here, $p_k$ counts only the retained columns (\cref{rem:bc_rows}). An orthogonal projector has a trace equal to its rank. Here, $\operatorname{trace}(Q_k) = p_k > MN_r \geq \operatorname{rank}(Q_k)$; therefore, $Q_k$ is not a projector. Since $Q_k$ is positive semidefinite and the average of its nonzero eigenvalues $p_k/\operatorname{rank}(Q_k) > 1$, at least one eigenvalue of $Q_k$ exceeds~$1$. Summing over the $S$ networks,
\begin{equation}
  \operatorname{trace}(K_P^{\mathrm{Adam}}) = \sum_{k} p_k = p,
  \qquad
  \operatorname{rank}(K_P^{\mathrm{Adam}}) \leq MN_r,
\end{equation}
thus
\begin{equation}
  \lambda_{\max}(K_P^{\mathrm{Adam}}) \geq p / MN_r > S .
\end{equation}
\end{proof}

\begin{remark}[Trace identity and conventions]
\label{rem:trace}\label{rem:bc_rows}
The trace identity explains the large $\lambda_{\max}$: $\operatorname{trace}(K_P^{\mathrm{Adam}}) = p$ versus $\operatorname{trace}(K_P^{\mathrm{GN}}) \ll p$, forcing $\lambda_{\max} \geq p/(MN_r) \gg S$. This identity is weight-invariant; thus, \cref{prop:adam} extends to any $\Lambda > 0$. The Jacobian is constructed using interior points. Zero columns from output biases are $O(1)$ in number and do not affect the bound.
\end{remark}

While \cref{prop:adam} concerns the coupling-independent structure of the normalized kernel, the residual-dynamics kernel additionally retains the coupling-scaled column magnitudes, and these reintroduce $\gamma$ into the stable learning rate.

\begin{proposition}[Coupling growth of Adam's residual-dynamics kernel]
\label{prop:adam_eta}
Under \cref{asm:segregated,asm:linear_coupling}, the residual-dynamics kernel \eqref{eq:adam_half_kernel} satisfies
\begin{equation}
  \lambda_{\max}\bigl(K_{1/2}^{\mathrm{Adam}}\bigr) \;\geq\; \gamma\, \nu(C_{ik}), \qquad \nu(C_{ik}) := \max_j \bigl\| [C_{ik}]_{:,j} \bigr\|_2 > 0,
  \label{eq:adam_eta_bound}
\end{equation}
and hence $\eta_{\max}^{\mathrm{Adam}} \leq 2 / (\gamma\, \nu(C_{ik})) = O(1/\gamma)$.
\end{proposition}

\begin{proof}
Fix the column index $j$ of network $k$ with $[C_{ik}]_{:,j} \neq 0$, and let $c_j[R_{i'}] \in \mathbb{R}^{N_r}$ denote the restriction of the column $c_j$ to the rows of residual $i'$. As $\theta_i \cap \theta_k = \emptyset$, residual $i$ depends on $\theta_{k,j}$ only through the coupling term $\gamma\, \mathcal{G}_i[\hat{u}_k]$. Thus,
\begin{equation}
  c_j[R_i] = \gamma\, [C_{ik}]_{:,j}.
\end{equation}
The squared norm of $c_j$ is a sum of non-negative row-block contributions:
\begin{equation}
  \|c_j\|^2 = \sum_{i'=1}^{M} \bigl\| c_j[R_{i'}] \bigr\|^2 \;\geq\; \bigl\| c_j[R_i] \bigr\|^2 = \gamma^2 \bigl\| [C_{ik}]_{:,j} \bigr\|^2 .
\end{equation}
Thus, $\|c_j\| \geq \gamma\, \| [C_{ik}]_{:,j} \|$. As $K_{1/2}^{\mathrm{Adam}}$ is a sum of positive semidefinite rank-one terms,
\begin{equation}
  K_{1/2}^{\mathrm{Adam}} \;\succeq\; \|c_j\|\, \tilde{c}_j \tilde{c}_j^\top ,
\end{equation}
whose largest eigenvalue is $\|c_j\|$. Maximizing over $j$ yields \eqref{eq:adam_eta_bound}.
\end{proof}

\begin{remark}[Spectral hierarchy]
\label{rem:adam_eta}
For all linearly coupled benchmarks, the Jacobian is affine in $\gamma$, yielding a matching upper bound $\lambda_{\max}(K_{1/2}^{\mathrm{Adam}}) = \Theta(\gamma)$ (\cref{app:adam_eta_proofs}); hence, $\eta_{\max}^{\mathrm{GD}} = O(1/\gamma^2) \ll \eta_{\max}^{\mathrm{Adam}} = O(1/\gamma) \ll 2/S \leq \eta_{\max}^{\mathrm{GN}}$.
\end{remark}

The $\Theta(\gamma)$ scaling above was derived under the modeling identification $\hat v_t \approx D$, which is eliminated by the following proposition.

Consider the one-way structure of \cref{sec:thermo}: with rows ordered $(\mathcal{R}_T, \mathcal{R}_u)$, the Jacobian satisfies $J_{T,\theta_u} = 0$ and $J_{u,\theta_T} = \gamma C_{uT}$ so that each $\theta_T$-column decomposes as $c_j = (\alpha_j, \gamma \beta_j)$ with $a_j := \|\alpha_j\|$, $b_j := \|\beta_j\|$, and $A := \sum_{j \in T} a_j^2$. We assume every $\theta_T$-column participates in the coupling, $b_j \geq b_{\min} > 0$, which is verified numerically for this system in \cref{sec:discussion}.

\begin{proposition}[Diagonal-class iteration-complexity separation, one-way coupling]
\label{prop:diag_lb}
Let $r_{t+1} = (I - \eta_t K^{(t)}) r_t$ with $K^{(t)} = J D_t^{-1} J^\top$ for an arbitrary sequence of positive diagonal matrices $D_t$---including Adam's realized second-moment sequence---and let $r_0$ be a unit residual supported on the $\mathcal{R}_T$ block. Then, $\|\Pi_T r_t\| \geq \langle r_0, r_t\rangle$, where $\Pi_T$ is the orthogonal projector onto the $\mathcal{R}_T$ row block of the residual vector, and
(i)~if $D_t \equiv D$ and $\eta_t \equiv \eta \leq 1/\lambda_{\max}(K^{(t)})$, then $\langle r_0, r_t\rangle \geq \bigl(1 - A/(\gamma^2 b_{\min}^2)\bigr)^t$; thus, for $\gamma^2 b_{\min}^2 \geq 2A$, halving the $\mathcal{R}_T$ component requires $t \geq \tfrac{\ln 2}{2}\,\gamma^2 b_{\min}^2/A = \Omega(\gamma^2)$ iterations, for every choice of $D$;
(ii)~for arbitrary $D_t$ and any stable steps $\eta_t \leq 2/\lambda_{\max}(K^{(t)})$, $\langle r_0, r_t\rangle \geq 1 - 2t\sqrt{A}/(\gamma b_{\min})$; thus, halving requires $t \geq \gamma b_{\min}/(4\sqrt{A}) = \Omega(\gamma)$.
\end{proposition}

\begin{proof}
As $J_{T,\theta_u} = 0$, every $\theta_u$ column has a vanishing $\mathcal{R}_T$-block. For any positive diagonal $D$ with diagonal entries $d_j$ and $r_0$ supported on $\mathcal{R}_T$,
\begin{equation}
  \Pi_T K \Pi_T = \sum_{j\in T} d_j^{-1}\, \alpha_j \alpha_j^\top .
\end{equation}
So,
\begin{equation}
  r_0^\top K r_0 \;\leq\; \operatorname{trace}(\Pi_T K \Pi_T) = \sum_{j\in T} a_j^2/d_j \;=:\; t_T.
\end{equation}
Retaining single rank-one terms,
\begin{equation}
  \lambda_{\max}(K)
  \;\geq\; \max_{j\in T} \frac{\|c_j\|^2}{d_j}
  \;\geq\; \gamma^2 b_{\min}^2 \max_{j\in T} d_j^{-1}
  \;\geq\; \gamma^2 b_{\min}^2 \sum_{j\in T} \frac{a_j^2}{A}\, d_j^{-1}
  \;=\; \frac{\gamma^2 b_{\min}^2}{A}\, t_T,
\end{equation}
as a maximum dominates every convex combination. Hence,
\begin{equation}
  t_T/\lambda_{\max}(K) \;\leq\; A/(\gamma^2 b_{\min}^2) \quad \text{for every positive diagonal};
  \label{eq:d_cancel}
\end{equation}
the choice of diagonal cancels.

\emph{Case~(i).}\; With fixed $K$ and $\eta \leq 1/\lambda_{\max}$, expand in the eigenbasis $\{(\mu_i, v_i)\}$:
\begin{equation}
  \langle r_0, r_t\rangle = \sum_i (1-\eta\mu_i)^t \langle r_0, v_i\rangle^2,
\end{equation}
with all factors in $[0,1]$. Convexity of $x \mapsto (1-\eta x)^t$ and Jensen's inequality yield
\begin{equation}
  \langle r_0, r_t\rangle \;\geq\; \bigl(1 - \eta\, r_0^\top K r_0\bigr)^t \;\geq\; \bigl(1 - t_T/\lambda_{\max}\bigr)^t,
\end{equation}
and \eqref{eq:d_cancel} applies. The halving count follows from $-\ln(1-x) \leq 2x$ for $x \leq 1/2$.

\emph{Case~(ii).}\; For $\eta_s \leq 2/\lambda_{\max}(K^{(s)})$, each factor $I - \eta_s K^{(s)}$ has a spectrum in $[-1,1]$; therefore, $\|r_s\| \leq 1$. Telescoping,
\begin{equation}
  1 - \langle r_0, r_t\rangle = \sum_{s<t} \eta_s \langle K^{(s)} r_0,\, r_s\rangle \;\leq\; \sum_{s<t} \eta_s \|K^{(s)} r_0\| .
\end{equation}
As $(K^{(s)})^2 \preceq \lambda_{\max}^{(s)} K^{(s)}$,
\begin{equation}
  \|K^{(s)} r_0\| \;\leq\; \bigl(\lambda_{\max}^{(s)}\, r_0^\top K^{(s)} r_0\bigr)^{1/2} \;\leq\; \bigl(\lambda_{\max}^{(s)}\, t_T^{(s)}\bigr)^{1/2},
\end{equation}
Combining with \eqref{eq:d_cancel}, $\eta_s \|K^{(s)} r_0\| \leq 2\,(t_T^{(s)}/\lambda_{\max}^{(s)})^{1/2} \leq 2\sqrt{A}/(\gamma b_{\min})$; therefore, $1 - \langle r_0, r_t\rangle \leq 2t\sqrt{A}/(\gamma b_{\min})$.
\end{proof}

By contrast, block-diagonal GN in the overparameterized regime yields $K_P = S I$ (equality case of \cref{thm:gn}, \cref{app:proofs}) so that every residual component contracts uniformly in $O(1)$ iterations at $\eta = \Theta(1/S)$.

\begin{corollary}[Weighted threshold, one-way coupling]
\label{cor:weighted}
In the setting of \cref{prop:diag_lb}, fix loss weights $\Lambda = \operatorname{blkdiag}(\lambda_T I, \lambda_u I)$ with $\lambda_T, \lambda_u > 0$. Expressed in the rescaled residual $\tilde r = \Lambda^{1/2} r$, the linearized weighted iteration reads $\tilde r_{t+1} = (I - \eta_t \tilde K^{(t)})\, \tilde r_t$ with $\tilde K^{(t)} = \Lambda^{1/2} J D_t^{-1} J^\top \Lambda^{1/2}$, and for every positive diagonal $D$,
\begin{equation}
  \tilde t_T / \lambda_{\max}(\tilde K) \;\le\; \frac{\lambda_T}{\lambda_u}\cdot\frac{A}{\gamma^2 b_{\min}^2},
  \qquad \tilde t_T := \operatorname{trace}(\Pi_T \tilde K \Pi_T).
  \label{eq:d_cancel_w}
\end{equation}
Consequently, cases (i) and (ii) of \cref{prop:diag_lb} hold with halving counts $\Omega(\gamma^2 \lambda_u/\lambda_T)$ and $\Omega(\gamma \sqrt{\lambda_u/\lambda_T})$, respectively; in particular, restoring the $O(1)$ iteration progress on $\mathcal{R}_T$ within the diagonal class requires $\lambda_T/\lambda_u = \Omega(\gamma^2)$.
\end{corollary}

\begin{proof}
The $\theta_T$-columns of $\Lambda^{1/2} J$ are $\tilde c_j = (\sqrt{\lambda_T}\,\alpha_j,\, \sqrt{\lambda_u}\,\gamma\beta_j)$ and the $\theta_u$-columns retain a vanishing $\mathcal{R}_T$-block so $\tilde t_T = \lambda_T \sum_{j\in T} a_j^2/d_j$; additionally, retaining single rank-one terms as before, $\lambda_{\max}(\tilde K) \geq \lambda_u \gamma^2 b_{\min}^2 \max_{j \in T} d_j^{-1} \geq \lambda_u \gamma^2 b_{\min}^2\, \tilde t_T/(\lambda_T A)$, which is \eqref{eq:d_cancel_w}. As $\Lambda^{1/2}$ is block-diagonal, a unit residual supported on $\mathcal{R}_T$ remains $\mathcal{R}_T$-supported after rescaling, and the Jensen and telescoping arguments of \cref{prop:diag_lb} apply with \eqref{eq:d_cancel} replaced by \eqref{eq:d_cancel_w}.
\end{proof}

\begin{remark}[Scope and consequences of the class bound]
\label{rem:diag_class}
The bound covers momentum-free descent. Classical acceleration preserves $\Omega(\gamma)$ separation in case~(i) \citep{nocedal2006numerical}. Unstable steps ($\eta_t > 2/\lambda_{\max}$) yield no descent guarantee. The proposition operates on the linearized residual dynamics. Reshaping the Jacobian based on feature learning lies outside its scope, although \cref{tab:training_stability} confirms that $\gamma$-separation persists at the trained parameter values. The one-way structure is essential: in reaction--diffusion, the rate constant enters the diagonal blocks corresponding to each residual's own network, eliminating the $\gamma$-light row block (\cref{tab:scope}). According to \cref{cor:weighted}, escaping through reweighting requires $\lambda_T/\lambda_u = \Omega(\gamma^2)$, which is outside the design objective of gradient-norm balancing (\cref{sec:balancing}). The obstruction is structural: diagonal scaling can equilibrate but not reshape the Gram matrix's two-block imbalance.
\end{remark}

\begin{table}[ht]
\caption{Scope summary ($\Lambda = \mathrm{blkdiag}(\lambda_i I_{N_r})$). Any $\Lambda > 0$: $\forall\,\lambda_i > 0$;\; uniform $\Lambda$: $\lambda_1 = \cdots = \lambda_M$.\ All bounds hold $\forall\,\theta$ (\cref{rem:scope}). Here $\sigma^2$ abbreviates $\sigma_{\max}^2(C_{ik})$, and $t_{1/2}(\mathcal{R}_T)$ denotes the number of iterations required to halve the $\mathcal{R}_T$ residual component; (i) and (ii) refer to the two cases of \cref{prop:diag_lb}, with case~(i) requiring $\gamma^2 b_{\min}^2 \geq 2A$.}
\label{tab:scope}
\centering
\resizebox{\textwidth}{!}{%
\begin{tabular}{@{}llccc@{}}
\toprule
Statement & Bound & Coupling & Loss weights $\Lambda$ & Arch. \\
\midrule
\cref{thm:gd} & $\lambda_{\max}(K) \geq \gamma^2 \sigma^2$ & Linear (\cref{asm:linear_coupling}) & Uniform $\Lambda$ & $S \!\geq\! 2$ \\
\cref{thm:gn} & $\lambda_{\max}(K_P) \leq S$ & Any & Any $\Lambda > 0$ (\cref{rem:scope}) & $S \!\geq\! 1$ \\
\cref{prop:soap_net} & $\lambda_{\max} \leq SL$ (\cref{asm:soap_approx}) & Any & Any $\Lambda > 0$ & $S \!\geq\! 1$ \\
\cref{prop:adam} & $\lambda_{\max} \geq p/(MN_r) \gg S$ & Any & Any $\Lambda > 0$ (\cref{rem:trace}) & $S \!\geq\! 1$ \\
\cref{prop:adam_eta} & $\Theta(\gamma)$ & Linear & Uniform $\Lambda$ & $S \!\geq\! 2$ \\
\cref{prop:diag_lb} & $t_{1/2}(\mathcal{R}_T) = \Omega(\gamma^2)$ (i); $\Omega(\gamma)$ (ii) & One-way linear & Uniform $\Lambda$ (thr.\ \cref{cor:weighted}) & $S \!\geq\! 2$ \\
\bottomrule
\end{tabular}}
\end{table}

\section{NTK Numerical Verification}
\label{sec:ntk_verification}

We verify the spectral predictions of \cref{sec:theory} by computing $\lambda_{\max}(K)$ and $\lambda_{\max}(K_P)$ at the network initialization stage across the 1D benchmark systems.

\subsection{Methodology}
\label{sec:ntk_methodology}

For each system, we computed the full Jacobian $J \in \mathbb{R}^{MN_r \times p}$ via per-sample backpropagation on interior collocation points. We formed five kernels using $J$, namely
$K = JJ^\top$ (standard NTK);
$K_P^{\mathrm{GN}} = \sum_k P_k$ via economy-size singular value decomposition of each $J_{\theta_k} = U_k \Sigma_k V_k^\top$, where $P_k = U_k U_k^\top$;
$K_P^{\mathrm{lw\text{-}GN}} = \sum_{k,l} P_k^{(l)}$, the layer-wise variant (\cref{prop:soap_net} under \cref{asm:soap_approx});
$K_P^{\mathrm{Adam}} = J D^{-1} J^\top$ with $D = \operatorname{diag}(J^\top J)$;
and $K_{1/2}^{\mathrm{Adam}} = J D^{-1/2} J^\top$.
All kernels were formed from the raw (unweighted) Jacobian, matching the uniform-$\Lambda$ scope of \cref{thm:gd,prop:adam_eta,prop:diag_lb}. The weight-invariant results (\cref{thm:gn,prop:adam}) were unaffected by this choice. Rank and clamping tolerances are listed in \cref{tab:implementation}.

\subsection{Spectral Verification at Initialization}
\label{sec:ntk_thermo}

All networks used four hidden layers with 64 neurons and $N_r = 200$ interior points. For the thermoelasticity benchmark, we performed a detailed $\gamma$-sweep ($\gamma \in \{0.1, 0.5, 1, 2, 5, 10, 25, 50, 100\}$, three seeds; \cref{fig:ntk_init}); for the RD ($k_{\mathrm{f}} = 0.5$ and $30$) and NP+P ($\varepsilon = 1.0$ and $0.1$) benchmarks, we compared the weak and strong coupling cases (\cref{tab:ntk_benchmarks}).

\begin{figure}[t]
  \centering
  \includegraphics[width=\textwidth]{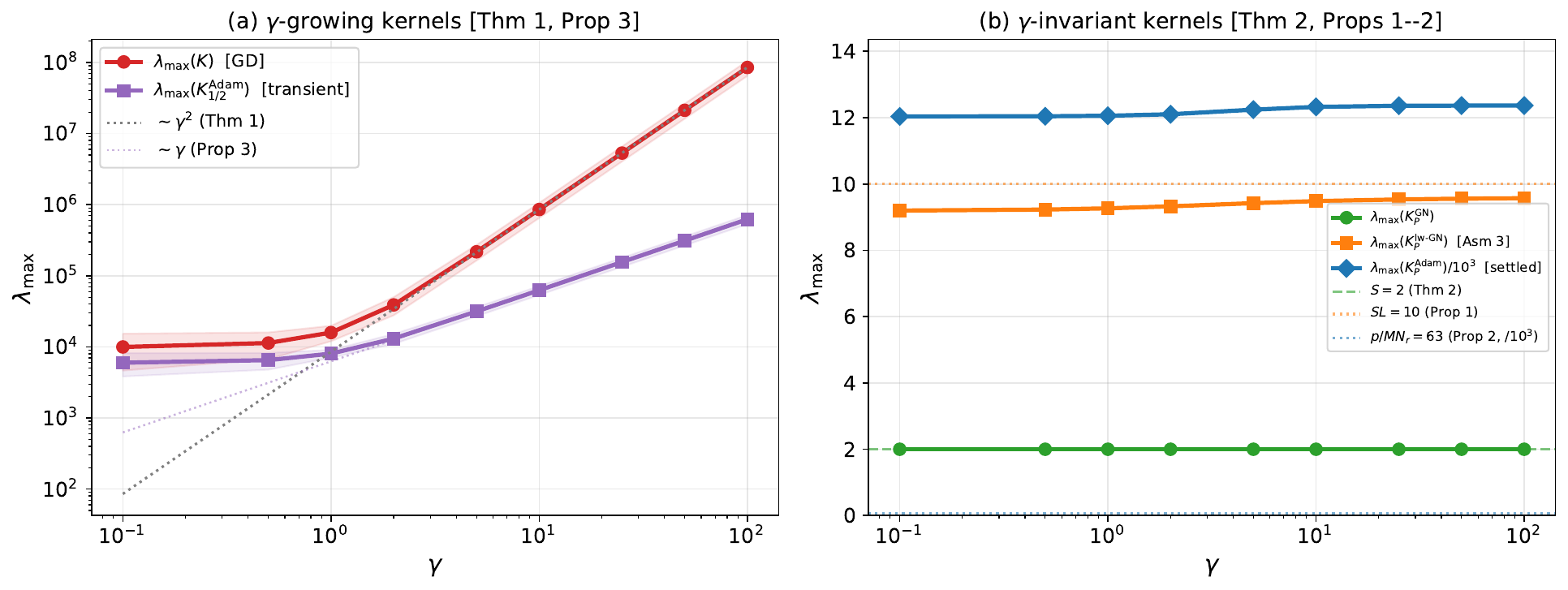}
  \caption{Neural tangent kernel (NTK) spectral analysis at initialization (thermoelasticity, $\gamma \in [0.1, 100]$, three seeds). \textbf{(a)}~$\gamma$-growing kernels (above the crossover $\gamma^{\ast} \approx 9$, \cref{app:single_net}): $\lambda_{\max}(K)$ grows as a function of $\gamma^2$ (\cref{thm:gd}) and $\lambda_{\max}(K_{1/2}^{\mathrm{Adam}})$ grows as a function of $\gamma$ (\cref{prop:adam_eta}). \textbf{(b)}~$\gamma$-invariant kernels: $\lambda_{\max}(K_P^{\mathrm{GN}}) = S = 2$ (\cref{thm:gn}), $\lambda_{\max}(K_P^{\mathrm{lw\text{-}GN}}) \approx 9.3 < SL = 10$ (\cref{prop:soap_net}, under \cref{asm:soap_approx}), and $\lambda_{\max}(K_P^{\mathrm{Adam}}) \approx 1.2 \times 10^{4} \gg p/MN_r \approx 63$ (\cref{prop:adam}); Adam's kernel is rescaled by $10^{-3}$ for visual comparisons. Thm: Theorem; Prop: Proposition; Asm: Assumption; GD: gradient descent.}
  \label{fig:ntk_init}
\end{figure}

\begin{table}[t]
  \caption{Neural tangent kernel (NTK) spectral measurements at initialization across the three 1D benchmark systems (three-seed mean). $S$ = number of networks.}
  \label{tab:ntk_benchmarks}
  \centering
  \small
  \resizebox{\textwidth}{!}{%
  \begin{tabular}{@{}llcccccc@{}}
    \toprule
    System & Coupling & $S$ & $\lambda_{\max}(K)$ & $\lambda_{\max}(K_P^{\mathrm{GN}})$ & $\lambda_{\max}(K_P^{\mathrm{Adam}})$ & $\lambda_{\max}(K_{1/2}^{\mathrm{Adam}})$ \\
    & & & & Bound: $S$ & Bound: $p/MN_r$ & \\
    \midrule
    Thermoelasticity & $\gamma = 0.1$ & 2 & 1.0e4 & 2.0 & 1.2e4 & 6.4e3 \\
    & $\gamma = 1$ & 2 & 1.8e4 & 2.0 & 1.2e4 & 8.8e3 \\
    & $\gamma = 10$ & 2 & 9.9e5 & 2.0 & 1.2e4 & 6.9e4 \\
    & $\gamma = 100$ & 2 & 9.8e7 & 2.0 & 1.2e4 & 6.8e5 \\
    \midrule
    Reaction--Diffusion (RD) & $k_{\mathrm{f}} = 0.5$ & 3 & 3.5e4 & 3.0 & 2.0e4 & 1.4e4 \\
    (asymmetric) & $k_{\mathrm{f}} = 30$ & 3 & 4.7e6 & 3.0 & 2.6e4 & 1.6e5 \\
    \midrule
    Nernst--Planck--Poisson (NP+P) & $\varepsilon = 1.0$ & 3 & 2.4e4 & 3.0 & 1.7e4 & 1.2e4 \\
    (nonlinear) & $\varepsilon = 0.1$ & 3 & 2.0e4 & 3.0 & 1.6e4 & 9.9e3 \\
    \bottomrule
  \end{tabular}}
\end{table}

Four observations are noted from \cref{fig:ntk_init,tab:ntk_benchmarks}: (1)~$\lambda_{\max}(K_P^{\mathrm{GN}}) = S$ exactly in all the systems---an algebraic identity given the full row rank of each $J_{\theta_k}$; the observed equality certifies the overparameterized regime assumed in \cref{prop:adam}. (2)~$\lambda_{\max}(K_P^{\mathrm{Adam}}) \gg S$ (${\sim}5300$--$8700 \times S$) and coupling-flat, confirming the projector-structure loss of \cref{prop:adam}. (3)~$\lambda_{\max}(K_{1/2}^{\mathrm{Adam}})$ grows with the coupling strength in the linearly coupled systems, consistent with the $\Theta(\gamma)$ scaling of \cref{prop:adam_eta}. (4)~In NP+P, neither $\lambda_{\max}(K)$ nor $\lambda_{\max}(K_{1/2}^{\mathrm{Adam}})$ grows with $1/\varepsilon$ at initialization; the coupling difficulty manifests during training (\cref{tab:training_stability}).

\subsection{Spectral Measurements During Training}

The bounds of \cref{sec:theory} hold at any $\theta$ value; thus, we report how the measured kernels behave at trained parameter values.

Across all the tested systems, $\lambda_{\max}(K_P^{\mathrm{GN}})$ remained exactly equal to $S$ at every measured snapshot, certifying that the overparameterization premise of \cref{prop:adam} persists throughout training given the full row rank of each $J_{\theta_k}$. \Cref{tab:training_stability} shows that $K_{1/2}^{\mathrm{Adam}}$ maintains $\gamma$-separated levels throughout training in the linearly coupled benchmarks. In the NP+P case, $K_{1/2}^{\mathrm{Adam}}$ grows more rapidly under strong coupling ($\varepsilon = 0.1$) during training, which we read as the state-dependent coupling emerging as the network approaches the physical solution (\cref{rem:gd_scope}).

\begin{table}[t]
  \caption{Training stability of $\lambda_{\max}(K_{1/2}^{\mathrm{Adam}})$ (single seed, Adam+GradNorm, 30000 epochs). The value of $\lambda_{\max}(K_P^{\mathrm{GN}})$ remains exactly at $S$ at every measured epoch in all systems. It is therefore omitted from the table. Values at initialization are listed in \cref{tab:ntk_benchmarks}.}
  \label{tab:training_stability}
  \centering
  \small
  \begin{tabular}{@{}llccc@{}}
    \toprule
    System & Coupling & $S$ & \multicolumn{2}{c}{$\lambda_{\max}(K_{1/2}^{\mathrm{Adam}})$} \\
    \cmidrule(l){4-5}
    & & & Epoch 14k & Epoch 30k \\
    \midrule
    Thermoelasticity & $\gamma = 1$ & 2 & 4.9e4 & 4.6e4 \\
    & $\gamma = 100$ & 2 & 5.2e5 & 4.0e5 \\
    \midrule
    Reaction--Diffusion & $k_{\mathrm{f}} = 0.5$ & 3 & 5.1e4 & 4.5e4 \\
    & $k_{\mathrm{f}} = 30$ & 3 & 1.2e5 & 8.8e4 \\
    \midrule
    NP+P & $\varepsilon = 1.0$ & 3 & 8.2e4 & 7.3e4 \\
    (nonlinear) & $\varepsilon = 0.1$ & 3 & 1.3e5 & 1.0e5 \\
    \bottomrule
  \end{tabular}
\end{table}

\section{Experimental Setup}

\Cref{tab:implementation} lists the full training configuration for the accuracy experiments of \cref{sec:results}: hardware and software versions, network sizes, optimizer settings, and run counts. The spectral measurements of \cref{sec:ntk_verification} use the separate configuration stated in \cref{sec:ntk_thermo} and summarized at the end of \cref{tab:implementation}.

\textbf{Architecture.}\; Each field variable was approximated by an independent multilayer perceptron (MLP) with disjoint parameters following the segregated-network approach adopted for electrokinetic multiphysics \citep{sun2024physics,merdasi2023physics}. \Cref{fig:architecture} illustrates the training pipeline. All networks use tanh activation and Xavier uniform initialization \citep{glorot2010understanding}. For 1D benchmarks, each network had four hidden layers with 64 neurons. For the 2D system, each network had five hidden layers with 128 neurons. Scalar fields ($\phi$, $c_+$, $c_-$) used single-output MLPs, while the flow network output $(u, v, p)$, yielding four networks ($S = 4$); the single-network variant used at $\varepsilon = 0.01$ (\cref{sec:2d_results}) instead used one network with six outputs.

\begin{figure}[t]
  \centering
  \includegraphics[width=\textwidth]{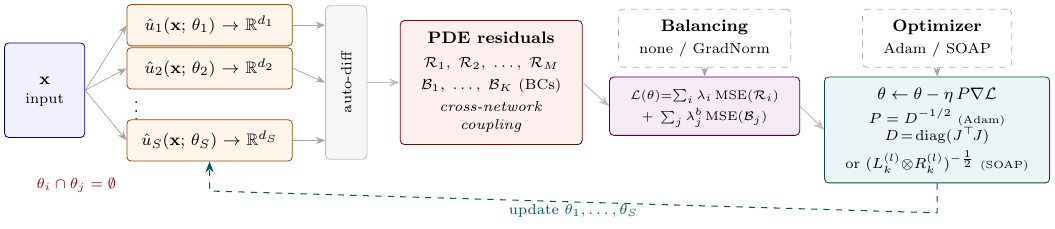}
  \caption{Segregated-network PINN architecture. Each field variable is approximated by an independent multilayer perceptron with disjoint parameters, so that no weight matrix is shared across fields. The inputs of the networks are the spatial coordinates, and the field variables are the outputs. Derivatives of every output are obtained by automatic differentiation (AD) and assembled into the PDE residuals and the boundary residuals. Coupling terms link a residual to the outputs of other networks (dashed arrows). The residual and boundary mean-squared errors are combined into the training loss, with the loss weights either fixed at one or set by inverse-gradient-norm balancing (GradNorm). The optimizer then updates all parameter groups simultaneously, using the diagonal preconditioner for Adam and a Kronecker-factored preconditioner for SOAP.}
  \label{fig:architecture}
\end{figure}

\textbf{Collocation.}\; Uniform sampling for Thermo and RD (200 interior points). For NP+P and 2D NP+P+Stokes, half uniform and half concentrated within a boundary-layer strip of thickness $\delta_{\mathrm{BL}} = \min(5\varepsilon,\, 0.3)$ (2D) or $\min(5\varepsilon,\, 0.5)$ (1D) adjacent to each wall. Interior points: 300 (1D); 3000 (2D). Boundary: 200 per edge (2D).

\textbf{Methods.}\; Four configurations were used for 1D (Adam, SOAP, Adam+GradNorm, SOAP+GradNorm) and two (SOAP+GradNorm, Adam+GradNorm) for 2D. The learning rate was equal to $10^{-3}$ (Thermo, RD) or $3\!\times\!10^{-4}$ (NP+P, 2D). Regarding the SOAP momentum parameters, $\beta = (0.99, 0.999)$ for all the systems, and the preconditioner update frequency was equal to every 2 steps. GradNorm: update frequency 1000 steps, exponential moving average momentum $= 0.9$. All optimizers shared identical learning rates, schedules, collocation sets, and epoch budgets within each benchmark, held fixed across coupling strengths so that accuracy differences reflect preconditioner structure and not per-configuration tuning.

\textbf{Training.}\; Different epochs were used for 1D (30000) and 2D (50000). Three seeds per configuration were used. In total, 210 runs were required for the main coupling-robustness study, and 12 additional runs for the $\varepsilon = 0.01$ architecture comparison (2 optimizers $\times$ 2 architectures $\times$ 3 seeds). All experiments were performed on a single NVIDIA RTX 4060 Ti (8\,GB) graphics processing unit (GPU).

\textbf{Metric.}\; Final-epoch relative $L_2$ error, averaged across field variables and seeds. Coupling degradation ratio: $\bar{L}_2(\text{strong}) / \bar{L}_2(\text{weak})$. For the 2D benchmark, $v$ and $p$ are excluded because the reference had $v^* \equiv 0$ and $p^* \equiv 0$, for which a relative $L_2$ error was undefined. The absolute deviations of the runs shown in \cref{fig:2d_adam_history,fig:2d_soap_history,fig:2d_adam_sn_history,fig:2d_soap_sn_history} are reported in the corresponding captions.

\section{Results}
\label{sec:results}

We report the results in three steps. \Cref{sec:main_result} measures how final-epoch accuracy scales with coupling strength across the three 1D benchmarks and establishes that SOAP+GradNorm is the only configuration whose degradation remains bounded in every regime tested. \Cref{sec:ablation} separates the two ingredients of this configuration. \Cref{sec:2d_results} then asks whether the same behavior survives when the system is scaled up, applying the configuration to the 2D six-residual electro-osmotic flow benchmark (down to $\varepsilon = 0.01$). Previous multiphysics-PINN studies managed coupling by restructuring the network or the variational formulation \citep{harandi2024mixed,haghighat2022physics,sun2024physics}. The preconditioner is an orthogonal axis, and the experiments below isolate it: architecture, learning rate, collocation set, and epoch budget are held fixed, and only the preconditioner and the balancing scheme vary.

\subsection{Main Result: Coupling Robustness}
\label{sec:main_result}

\Cref{tab:degradation} reports coupling degradation ratios in the final-epoch metric. Full per-configuration means and standard deviations appear in \cref{app:results}. \Cref{tab:absolute} lists the underlying absolute errors at the weakest and strongest coupling of each system.

\begin{table}[t]
  \caption{Coupling degradation (Final $L_2$, three-seed mean). FAIL: Final $L_2 > 0.1$ at the strongest coupling.}
  \label{tab:degradation}
  \centering
  \small
  \begin{tabular}{@{}lccc@{}}
    \toprule
    Method & Thermo ($\gamma\!:\!1{\to}100$) & RD ($k\!:\!1{\to}100$) & NP+P ($\varepsilon\!:\!1{\to}0.1$) \\
    \midrule
    Adam      & $16.0\times$ & $9.6\times$  & FAIL \\
    SOAP      & $1.1\times$  & $1.5\times$  & $118\times$ \\
    Adam+GradNorm   & FAIL & $40.4\times$ & FAIL \\
    \textbf{SOAP+GradNorm} & $\mathbf{0.6\times}$ & $\mathbf{0.9\times}$ & $\mathbf{2.3\times}$ \\
    \bottomrule
  \end{tabular}
\end{table}

\begin{table}[t]
  \caption{Absolute Final $L_2$ at weak and strong coupling (three-seed mean).}
  \label{tab:absolute}
  \centering
  \small
  \begin{tabular}{@{}lcccccc@{}}
    \toprule
    Method & \multicolumn{2}{c}{Thermo} & \multicolumn{2}{c}{RD} & \multicolumn{2}{c}{NP+P} \\
    \cmidrule(lr){2-3} \cmidrule(lr){4-5} \cmidrule(lr){6-7}
    & $\gamma\!=\!1$ & $\gamma\!=\!100$ & $k\!=\!1$ & $k\!=\!100$ & $\varepsilon\!=\!1$ & $\varepsilon\!=\!0.1$ \\
    \midrule
    Adam     & 3.6e-3 & 5.7e-2 & 2.0e-3 & 2.0e-2 & 1.4e-3 & FAIL \\
    SOAP     & 9.0e-5 & 9.7e-5 & 3.3e-5 & 4.7e-5 & 2.9e-5 & 3.4e-3 \\
    Adam+GradNorm  & 9.5e-4 & FAIL   & 1.2e-3 & 4.7e-2 & 2.0e-3 & FAIL \\
    SOAP+GradNorm & \textbf{1.2e-4} & \textbf{7.4e-5} & \textbf{2.9e-5} & \textbf{2.7e-5} & \textbf{1.8e-5} & \textbf{4.2e-5} \\
    \bottomrule
  \end{tabular}
\end{table}

The bound $\lambda_{\max}(K_P^{\mathrm{GN}})$ is independent of the coupling strength (\cref{sec:theory}), and the accuracy profile of SOAP+GradNorm (green) is likewise flat in \cref{fig:l2_vs_coupling}.

\begin{figure}[t]
  \centering
  \includegraphics[width=\textwidth]{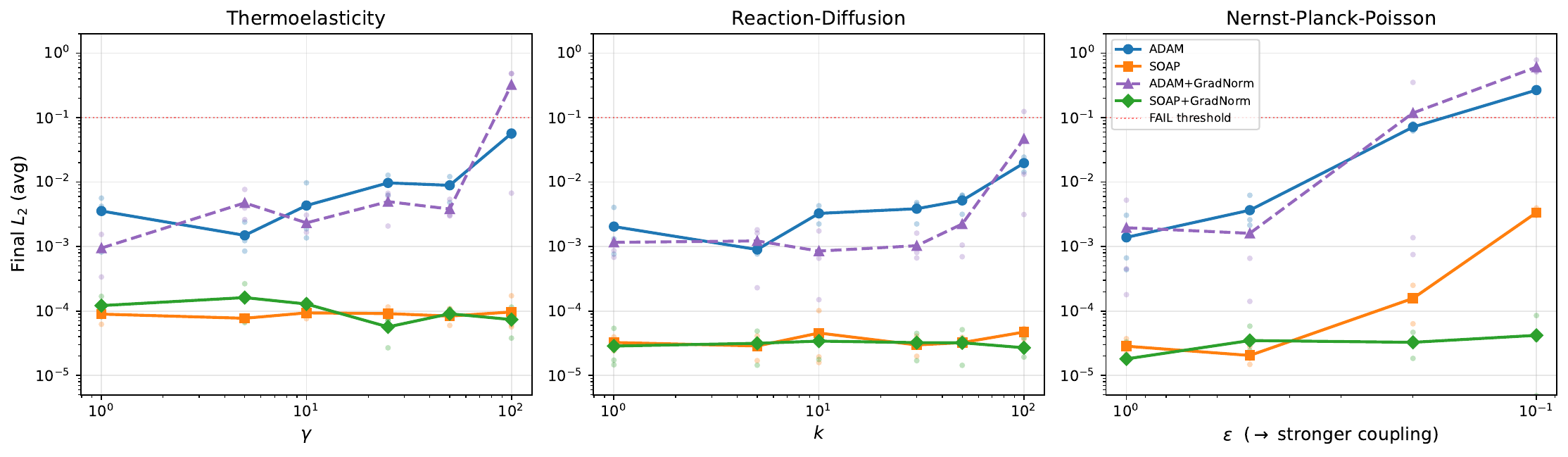}
  \caption{Final $L_2$ error vs.\ coupling strength. SOAP+GradNorm (green) remains flat across coupling strengths, matching the coupling-independent spectral bound of the preconditioned NTK (\cref{thm:gn}).}
  \label{fig:l2_vs_coupling}
\end{figure}

\subsection{Ablation}
\label{sec:ablation}

\textbf{Preconditioning (SOAP).}\; At the strongest coupling in each linear system, replacing Adam with SOAP (both unbalanced) improves Final $L_2$ by over two orders of magnitude (\cref{tab:absolute}), consistent with the spectral advantage of block-diagonal over diagonal preconditioning (\cref{thm:gn,prop:adam}).

\textbf{Gradient balancing (GradNorm).}\; In linear systems, SOAP alone maintains a modest degradation (\cref{tab:degradation}). In nonlinear NP+P, the final $L_2$ error of SOAP alone degrades by $118\times$; adding GradNorm reduces this to $2.3\times$.

\subsection{2D Electro-osmotic Flow}
\label{sec:2d_results}

\begin{table}[t]
  \caption{2D NP+P+Stokes Final $L_2$ ($\zeta = 1.0$, three-seed mean $\pm$ standard deviation). Rows for $\varepsilon \geq 0.05$ use the segregated architecture ($S = 4$). At $\varepsilon = 0.01$, both segregated and single-network ($S = 1$) architectures are compared.}
  \label{tab:2d_npps}
  \centering
  \small
  \begin{tabular}{@{}llccc@{}}
    \toprule
    $\varepsilon$ & Architecture & SOAP+GradNorm & Adam+GradNorm & Ratio \\
    \midrule
    0.2  & Segregated ($S\!=\!4$) & 2.68e-4 $\pm$ 2.0e-4 & FAIL ($0.32 \pm 0.07$) & $1202\times$ \\
    0.1  & Segregated ($S\!=\!4$) & 5.12e-4 $\pm$ 1.0e-4 & FAIL ($0.32 \pm 0.06$) & $623\times$ \\
    0.05 & Segregated ($S\!=\!4$) & 2.40e-3 $\pm$ 2.2e-3 & FAIL ($0.47 \pm 0.17$) & $194\times$ \\
    \midrule
    0.01 & Segregated ($S\!=\!4$) & FAIL ($0.66 \pm 0.38$) & FAIL ($1.27 \pm 0.02$) & --- \\
    0.01 & Single ($S\!=\!1$) & \textbf{1.30e-2 $\pm$ 5.2e-3} & FAIL ($2.59 \pm 1.25$) & $199\times$ \\
    \bottomrule
  \end{tabular}
\end{table}

Adam+GradNorm fails at all $\varepsilon$ (all values above the $0.1$ failure threshold), while SOAP+GradNorm achieves an accuracy of $10^{-4}$ at $\varepsilon = 0.2$ with the segregated architecture. At $\varepsilon = 0.01$---a regime where the Debye length $\lambda_D = \varepsilon$ reduces to $1\%$ of the channel width---the segregated architecture reaches its limit (SOAP+GradNorm: $L_2 = 0.66$); however, switching to a single-network architecture ($S = 1$) recovers $L_2 = 1.3 \times 10^{-2}$, a ${\sim}50\times$ improvement; because the segregated baseline did not converge, this ratio marks a transition from failure to success, not a refinement of an already converged solution. Adam+GradNorm fails under both architectures; therefore, the architectural change does not recover it. Training histories and field-error maps for both optimizers are presented in \cref{app:2d_training}. Seed-to-seed spread is substantial at all $\varepsilon$ (\cref{tab:2d_npps}); therefore, the robust observations are the two-to-three-order-of-magnitude gaps between optimizers, not the individual values.

\section{Discussion}
\label{sec:discussion}

The mechanism behind coupling-robust accuracy is best understood by connecting the theory of \cref{sec:theory} to the empirical trends of \cref{sec:results}. We first explain why Adam's diagonal preconditioning fails to protect against coupling, whereas block-diagonal preconditioning does, and how this failure manifests both in the stable learning rate and in the per-residual distribution of convergence. We then interpret two empirical findings through this lens---how the preferred architecture interacts with the optimizer at $\varepsilon = 0.01$ and why the $\varepsilon$-controlled NP+P system retains residual degradation despite the coupling-independent spectral bound. Relative to prior multiphysics-PINN strategies that manage coupling architecturally---mixed formulations \citep{harandi2024mixed}, sequential field splitting \citep{haghighat2022physics}, and segregated networks \citep{sun2024physics}---our results identify preconditioner structure as an additional and largely independent lever; within a fixed architecture, changing only the preconditioner shifts the strong-coupling error by two to three orders of magnitude. The two axes also interact, as the preferred architecture itself depends on the optimizer (see below). Therefore, we do not claim that preconditioning supersedes architectural design. We close with the computational cost of this choice and the limitations of the analysis.

\textbf{Reasons why Adam does not resolve coupling.}\; Diagonal preconditioning destroys the projector structure that maintains $\lambda_{\max}(K_P) \leq S$ (\cref{prop:adam}). Under linear coupling it also shrinks the stable learning rate as $O(1/\gamma)$ (\cref{prop:adam_eta}). The failures of Adam+GradNorm across all the benchmarks (\cref{tab:degradation}) are consistent with both. The data indicate that under identical learning rates, SOAP+GradNorm remains flat across coupling strengths while Adam degrades, a differential response that scalar step-size rescaling cannot produce. At $\gamma = 100$, plain Adam converges ($L_2 = 5.7\times10^{-2}$) while Adam+GradNorm fails, indicating that an interaction with adaptive weighting is responsible rather than step-size excess. \Cref{prop:adam} concerns a structural limitation, not final accuracy; the bound $\lambda_{\max} \leq S$ is unavailable to a diagonal preconditioner at any $\theta$, for any coupling type and any loss weighting.

\textbf{Per-residual convergence under the residual-dynamics kernel.}\; The argument below is a mechanism sketch, not a formal result; it relies on a column condition verified numerically for this system. Beyond shrinking the stable learning rate, the coupling-scaled column weights underlying \cref{prop:adam_eta} also redistribute spectral energy across residuals; this directional effect can be quantified directly on the residual-dynamics kernel $K_{1/2}^{\mathrm{Adam}} = \sum_j \|c_j\|\,\tilde{c}_j\tilde{c}_j^\top$. Define the per-residual trace, where $R_\alpha$ denotes the row block of residual $\alpha$,
\begin{equation}
  \operatorname{trace}_{R_\alpha}\!\bigl(K_{1/2}^{\mathrm{Adam}}\bigr) := \sum_j \|c_j\|\, \|\tilde{c}_j[R_\alpha]\|^2 .
  \label{eq:per_block_trace}
\end{equation}
For the thermoelasticity Jacobian structure, $\theta_T$-columns satisfy $\|\tilde{c}_j[R_T]\|^2 = a_j^2/(a_j^2 + \gamma^2 b_j^2)$ (with $a_j$ and $b_j$ as defined in \cref{prop:diag_lb}), while $\theta_u$-columns have $\|\tilde{c}_j[R_T]\|^2 = 0$ exactly. Assuming that every retained $\theta_T$ column enters the coupling term ($b_j > 0$ for all $j \in T$; a column with $b_j = 0$ would keep an $O(1)$ weight, and no such column occurs numerically in this system):
\begin{equation}
  \operatorname{trace}_{R_T}\!\bigl(K_{1/2}\bigr) = \sum_{j \in T} \frac{a_j^2}{\|c_j\|} \leq \sum_{j \in T} \frac{a_j^2}{\gamma\, b_j} = O(1/\gamma),
  \label{eq:trace_RT_decay}
\end{equation}
while $\operatorname{trace}(K_{1/2}) = \sum_j \|c_j\| = \Theta(\gamma)$ (\cref{rem:adam_eta}); therefore, the $R_T$ fraction decays as $O(1/\gamma^2)$. The standard eigenvector bound then applies to $K_{1/2}$ itself: for the top eigenvector $v_1$ (with eigenvalue $\lambda_1 = \Theta(\gamma)$),
\begin{equation}
  \|v_1[R_T]\|^2 \;\leq\; \frac{\operatorname{trace}_{R_T}(K_{1/2})}{\lambda_1} \;=\; \frac{O(1/\gamma)}{\Theta(\gamma)} \;=\; O(1/\gamma^2) \;\to\; 0.
  \label{eq:eigvec_bound}
\end{equation}
Numerically ($\gamma = 10$), the $R_T$ fraction is $0.3\%$ in $K_{1/2}^{\mathrm{Adam}}$, in line with the $O(1/\gamma^2)$ decay derived above. This exclusion is invariant to scalar learning-rate rescaling (which preserves eigenvector composition) and is not removed by reweighting within the diagonal class (\cref{cor:weighted}). Under linearized residual dynamics, the convergence time for $R_T$ is governed by the ratio of the top eigenvalue to that of the slow directions carrying its energy, which grows large when that energy concentrates in slow eigendirections. GN eliminates this mechanism as the preconditioned kernel is a multiple of the identity in the overparameterized regime (equality case of \cref{thm:gn}). For the one-way case, this exclusion is the finite-time signature of the class-level bound of \cref{prop:diag_lb}, according to which no diagonal reweighting or adaptation restores the $O(1)$ iteration progress on $\mathcal{R}_T$. Although \cref{prop:diag_lb} is proved only for the one-way case, the same pattern appears in the bidirectionally coupled RD benchmark; as $k_{\mathrm{f}}\!: 1 \to 100$, the dominant-eigenvector share of $R_A$ decreases from $22.3\%$ to $5.0\%$.

\textbf{Architecture selection.}\;
\label{sec:discussion_arch}
In the 2D NP+P+Stokes benchmark ($\varepsilon = 0.01$), SOAP+GradNorm with a single shared network achieves $L_2 = 0.013$, a ${\sim}50\times$ improvement over the segregated design ($L_2 = 0.66$), a direction consistent with the layer-wise bound tightening from $\lambda_{\max}(K_P^{\mathrm{SOAP}}) \leq SL$ to $\leq L$ (\cref{prop:soap_net,rem:single_net}). Adam shows the opposite trend ($L_2$: $1.27 \to 2.59$), with both values above the failure threshold. This comparison rests on a single operating point ($\varepsilon = 0.01$, 2 optimizers $\times$ 2 architectures $\times$ 3 seeds), and three of the four cells did not converge. The architectural reading is therefore indicative only.

\textbf{Coupling vs.\ stiffness in NP+P.}\; The parameter $\varepsilon$ controls both the coupling strength and solution stiffness. In the nondimensionalized Poisson equation (\cref{sec:benchmarks}), $\varepsilon$ is the Debye length; thus, reducing $\varepsilon$ from $1.0$ to $0.1$ simultaneously narrows the $O(\varepsilon)$ boundary layer and increases the effective coupling coefficient $1/\varepsilon^2$ by two orders of magnitude. The preconditioned NTK bound (\cref{thm:gn}) guarantees that $\lambda_{\max}(K_P^{\mathrm{GN}}) \leq S$ independently of $\varepsilon$, yet SOAP+GradNorm still exhibits $2.3\times$ degradation in NP+P (\cref{tab:degradation})---worse than the near-zero degradation in linear systems. This residual gap likely reflects the approximation challenge of resolving $O(\varepsilon)$ boundary layers and the gap between linearized NTK dynamics and actual training. A precise characterization remains open.

\textbf{Computational cost.}\; SOAP incurs a wall-clock overhead per epoch of ${\sim}1.6\times$ compared with Adam on a single NVIDIA RTX 4060 Ti GPU owing to Kronecker factor maintenance. Given the substantial accuracy improvements documented in \cref{tab:absolute}, this overhead is modest.

\textbf{Limitations.}\; The spectral bounds operate on linearized residual dynamics. Their connection to final-epoch accuracy is supported empirically (\cref{tab:degradation,tab:absolute}) but not by a convergence theorem. The transient-phase results (\cref{thm:gd,prop:adam_eta,prop:diag_lb}) assume uniform loss weights. A general weighted extension beyond the one-way case (\cref{cor:weighted}) constitutes future work. The 2D benchmark is $x$-invariant by design (\cref{sec:eof}); therefore, whether the coupling robustness reported here persists under axially varying boundary conditions remains untested. \Cref{asm:soap_approx} states a layer-wise correspondence between SOAP's Kronecker factors and the GN matrix. It is not verified numerically for the realized preconditioner; therefore, \cref{prop:soap_net} describes the exact layer-wise GN form, not SOAP's running-average approximation.

\section{Conclusion}

We have shown that the coupling-induced accuracy degradation of multiphysics PINNs is governed by the spectral structure of the preconditioned NTK and that a Kronecker-preconditioned optimizer with gradient-norm balancing keeps the resulting degradation bounded across a range of systems. Block-diagonal GN preconditioning bounds $\lambda_{\max}(K_P) \leq S$ independent of coupling strength, PDE structure, and network parameterization, while gradient descent ($\Omega(\gamma^2)$) and Adam ($\Theta(\gamma)$ in the transient phase; $p/(MN_r) \gg S$ in the settled phase) retain $\gamma$- or dimension-dependent spectral radii. For one-way coupling, this gap is class-wide: no diagonal preconditioner matches block GN's $O(1)$ iteration complexity on the driving residual (\cref{prop:diag_lb}).

Across 222 experiments, SOAP+GradNorm is the only configuration whose degradation remains bounded in every regime tested and the only one that succeeds on the 2D six-residual system, preserving weak-coupling accuracy in linear systems and limiting degradation to $2.3\times$ in nonlinear NP+P (\cref{tab:degradation}). Scaled up to a 2D, six-residual EOF system ($\varepsilon = 0.01$), SOAP+GradNorm with a single-network architecture still converges ($L_2 = 1.3 \times 10^{-2}$) with the Debye layer resolved, on an $x$-invariant reference. More broadly, these results reframe coupling-induced accuracy loss in multiphysics PINNs as a problem of preconditioner structure rather than of loss weighting: the decisive lever is whether curvature is inverted block-wise across fields or merely rescaled diagonally. This shift replaces case-by-case loss-weight tuning with a structural criterion for training PINNs on strongly coupled, stiff systems such as EDL-resolved electrokinetics. Extensions to time-dependent multiphysics, genuinely 2D geometries, and convergence guarantees beyond spectral conditioning constitute the scope of future work.

\section*{Data statement}

The code and data supporting the findings of this study are openly available at \url{https://github.com/YoungjaePark99/PINNmultiphysicsSOAP}.

\bibliographystyle{unsrtnat}
\bibliography{references}

\newpage
\appendix

\section{Proofs}
\label{app:proofs}

\subsection{Proof of \cref{thm:gd}}
The $(i,i)$-th diagonal block of $K = JJ^\top$ is $K_{[i,i]} = \sum_{j=1}^{S} J_{i,\theta_j} J_{i,\theta_j}^\top$. As every term in the sum is positive semidefinite, $K_{[i,i]} \succeq J_{i,\theta_k} J_{i,\theta_k}^\top = \gamma^2 C_{ik} C_{ik}^\top$, hence $\lambda_{\max}(K_{[i,i]}) \geq \gamma^2 \sigma_{\max}^2(C_{ik})$. Given that $K_{[i,i]}$ is a principal submatrix of the positive semidefinite matrix $K \in \mathbb{R}^{MN_r \times MN_r}$, $\lambda_{\max}(K) \geq \lambda_{\max}(K_{[i,i]})$ follows from the variational characterization of eigenvalues (any unit vector in the subspace can be zero-padded to a unit vector in the full space without changing its Rayleigh quotient), thus completing the proof.

\subsection{Proof of \cref{lem:projector}}
From the block structure of $H$, $K_P = J H^{+} J^\top = \sum_k J_{\theta_k} H_k^{+} J_{\theta_k}^\top$. For each term, $J_{\theta_k} (J_{\theta_k}^\top J_{\theta_k})^{+} J_{\theta_k}^\top$ is the orthogonal projector onto $\mathrm{col}(J_{\theta_k})$ (standard result from matrix analysis).

\subsection{Equality Condition in \cref{thm:gn}}
If $\mathbf{v} \in \mathrm{col}(J_{\theta_k})$ for all $k$, then $P_k \mathbf{v} = \mathbf{v}$ and $\|P_k \mathbf{v}\| = 1$ for all $k$, achieving $\mathbf{v}^\top K_P \mathbf{v} = S$. In the overparameterized regime ($p_k > MN_r$ with $J_{\theta_k}$ full row rank), $\mathrm{col}(J_{\theta_k}) = \mathbb{R}^{MN_r}$ for all $k$, so $P_k = I$ and $K_P = SI$.

\subsection{Proof of \cref{prop:soap_net}}

\emph{Part (1).}\; The layer-level index set $\{(k,l) : k=1,\ldots,S,\; l=1,\ldots,L\}$ is a refinement of the network-level partition $\{k : k=1,\ldots,S\}$; each network group $\{\theta_k\}$ is equal to the union $\bigcup_{l=1}^{L} \{\theta_k^{(l)}\}$. A matrix that is block-diagonal with respect to a finer partition is block-diagonal with respect to any coarser partition. As $H^{\mathrm{SOAP}}$ has zero blocks between any two distinct index pairs $(k,l) \neq (k',l')$, it has zero blocks between distinct networks $k \neq k'$.

\emph{Part (2).}\; From the block-diagonal structure, $K_P^{\mathrm{SOAP}} = \sum_{k,l} P_k^{(l)}$, where $P_k^{(l)} = J_{\theta_k^{(l)}}(\tilde{H}_k^{(l)})^+ J_{\theta_k^{(l)}}^\top$ is the orthogonal projector onto $\mathrm{col}(J_{\theta_k^{(l)}})$ (when the Kronecker approximation is exact). For any unit vector~$\mathbf{v}$,
\begin{equation}
  \mathbf{v}^\top K_P^{\mathrm{SOAP}}\, \mathbf{v} = \sum_{k=1}^{S}\sum_{l=1}^{L} \|P_k^{(l)}\mathbf{v}\|^2 \leq \sum_{k=1}^{S}\sum_{l=1}^{L} 1 = SL,
\end{equation}
where $P_k^{(l)}$ is the orthogonal projector onto the column space of the layer-wise Jacobian block $J_{\theta_k^{(l)}}$, $S$ is the number of networks, and $L$ is the number of layers per network.

\subsection{\cref{prop:adam_eta}: Conventions, Single-network Variant, and Upper Bound}
\label{app:adam_eta_proofs}

\emph{Zero columns.} Columns with $D_{jj} = 0$ (parameters whose interior-residual gradient vanishes identically, e.g., output biases annihilated by every differential operator in the system) do not contribute to either kernel under any regularization convention. The sums in \cref{eq:adam_half_kernel,eq:adam_norm_kernel} are restricted accordingly.

\emph{Matching upper bound.} For all linearly coupled benchmarks in this study, the residuals are affine in the coupling parameter, as is the Jacobian: $J = J^{(0)} + \gamma\, J^{(1)}$, with $c_j^{(0)}$ and $c_j^{(1)}$ denoting the $j$-th columns of $J^{(0)}$ and $J^{(1)}$, respectively. Then, $\|c_j\| \leq \|c_j^{(0)}\| + \gamma\, \|c_j^{(1)}\|$ for every column, and
\begin{equation}
  \lambda_{\max}\bigl(K_{1/2}^{\mathrm{Adam}}\bigr) \;\leq\; \operatorname{trace}\bigl(K_{1/2}^{\mathrm{Adam}}\bigr) = \sum_j \|c_j\| \;\leq\; T_0 + \gamma\, T_1,
  \qquad T_0 := \sum_j \bigl\|c_j^{(0)}\bigr\|,\quad T_1 := \sum_j \bigl\|c_j^{(1)}\bigr\|,
\end{equation}
Therefore, $\lambda_{\max}(K_{1/2}^{\mathrm{Adam}}) = \Theta(\gamma)$. The trace identity follows from $\operatorname{trace}(J D^{-1/2} J^\top) = \operatorname{trace}(D^{-1/2} J^\top J) = \sum_j [J^\top J]_{jj} / \sqrt{D_{jj}} = \sum_j \sqrt{D_{jj}} = \sum_j \|c_j\|$.

\subsection{Single-network Extensions}
\label{app:single_net}

\emph{Single-network variant.} Based on the single-network hypothesis of \cref{rem:single_net}, coupling enters additively within each row block: $J_i = A_i + \gamma\, C_i$ with $A_i, C_i$ independent of $\gamma$ and $C_i \neq 0$. Fix $j$ with $[C_i]_{:,j} \neq 0$. Restricting the column $c_j$ to the rows of residual $i$ and applying the reverse triangle inequality yields
\begin{equation}
  \|c_j\| \;\geq\; \bigl\| [A_i]_{:,j} + \gamma\, [C_i]_{:,j} \bigr\| \;\geq\; \gamma\, \bigl\| [C_i]_{:,j} \bigr\| - \bigl\| [A_i]_{:,j} \bigr\| \;\geq\; \tfrac{\gamma}{2}\, \bigl\| [C_i]_{:,j} \bigr\|
  \qquad \text{for } \gamma \geq 2\gamma_j^{\ast} := 2\, \frac{\bigl\| [A_i]_{:,j} \bigr\|}{\bigl\| [C_i]_{:,j} \bigr\|} .
\end{equation}
The rank-one argument of \cref{prop:adam_eta} yields $\lambda_{\max}(K_{1/2}^{\mathrm{Adam}}) \geq \tfrac{\gamma}{2}\, \nu(C_i)$ above the columnwise crossover---mirroring the relation between \cref{thm:gd} and its single-network extension (\cref{rem:single_net}). On the thermoelasticity initialization data, the columnwise crossover $\gamma^{\ast} := \max_j \gamma_j^{\ast}$ evaluates to $\gamma^{\ast} \approx 9$, consistent with the muted $\gamma\!: 1 \to 10$ ratios in \cref{fig:ntk_init}.

\emph{GD bound.}\; $\sigma_{\max}(J_i) \geq \gamma\,\sigma_{\max}(C_i) - \sigma_{\max}(A_i)$, hence $\lambda_{\max}(K) \geq (\gamma\,\sigma_{\max}(C_i) - \sigma_{\max}(A_i))^2$ for $\gamma \geq \gamma_0 := \sigma_{\max}(A_i)/\sigma_{\max}(C_i)$.

\emph{GN bound.}\; $K_P^{\mathrm{GN}} = P_1$, $\lambda_{\max} \leq 1$.

\emph{SOAP bound.}\; $K_P^{\mathrm{SOAP}} = \sum_l P^{(l)}$, $\lambda_{\max} \leq L$.

\section{Full Final $L_2$ Results}
\label{app:results}

All reported values are the mean $\pm$ sample standard deviation over three seeds. Deg denotes the coupling degradation ratio.

\subsection{Thermoelasticity --- Final $L_2$}

\begin{table}[H]
  \centering
  \resizebox{\textwidth}{!}{%
  \begin{tabular}{@{}lccccccccc@{}}
    \toprule
    Method & $\gamma\!=\!1$ & CV & $\gamma\!=\!5$ & $\gamma\!=\!10$ & $\gamma\!=\!25$ & $\gamma\!=\!50$ & $\gamma\!=\!100$ & CV & Deg \\
    \midrule
    Adam           & 3.56e-03$\pm$2.48e-03 & 70\% & 1.49e-03$\pm$8.05e-04 & 4.34e-03$\pm$4.71e-03 & 9.73e-03$\pm$3.27e-03 & 8.93e-03$\pm$3.75e-03 & 5.69e-02$\pm$5.29e-03 & 9\% & $16.0\times$ \\
    SOAP           & 8.95e-05$\pm$2.39e-05 & 27\% & 7.72e-05$\pm$7.87e-06 & 9.36e-05$\pm$1.59e-05 & 9.12e-05$\pm$2.11e-05 & 8.34e-05$\pm$2.47e-05 & 9.66e-05$\pm$6.57e-05 & 68\% & $1.1\times$ \\
    Adam+GradNorm  & 9.48e-04$\pm$6.04e-04 & 64\% & 4.76e-03$\pm$2.63e-03 & 2.33e-03$\pm$7.11e-04 & 4.98e-03$\pm$2.53e-03 & 3.83e-03$\pm$1.38e-03 & FAIL ($0.32$) &  & FAIL \\
    SOAP+GradNorm  & 1.21e-04$\pm$4.42e-05 & 36\% & 1.62e-04$\pm$9.98e-05 & 1.28e-04$\pm$1.86e-05 & 5.67e-05$\pm$2.58e-05 & 9.06e-05$\pm$1.43e-05 & 7.37e-05$\pm$3.89e-05 & 53\% & $0.6\times$ \\
    \bottomrule
  \end{tabular}}
\end{table}

\subsection{Reaction--Diffusion --- Final $L_2$}

\begin{table}[H]
  \centering
  \resizebox{\textwidth}{!}{%
  \begin{tabular}{@{}lccccccccc@{}}
    \toprule
    Method & $k\!=\!1$ & CV & $k\!=\!5$ & $k\!=\!10$ & $k\!=\!30$ & $k\!=\!50$ & $k\!=\!100$ & CV & Deg \\
    \midrule
    Adam           & 2.04e-03$\pm$1.76e-03 & 86\% & 9.00e-04$\pm$1.74e-04 & 3.27e-03$\pm$1.03e-03 & 3.87e-03$\pm$1.41e-03 & 5.20e-03$\pm$1.76e-03 & 1.97e-02$\pm$5.06e-03 & 26\% & $9.6\times$ \\
    SOAP           & 3.25e-05$\pm$7.73e-06 & 24\% & 2.87e-05$\pm$1.20e-05 & 4.56e-05$\pm$4.84e-05 & 2.97e-05$\pm$1.05e-05 & 3.24e-05$\pm$5.91e-06 & 4.72e-05$\pm$9.44e-06 & 20\% & $1.5\times$ \\
    Adam+GradNorm  & 1.16e-03$\pm$6.83e-04 & 59\% & 1.22e-03$\pm$8.66e-04 & 8.52e-04$\pm$8.15e-04 & 1.03e-03$\pm$5.12e-04 & 2.24e-03$\pm$2.37e-03 & 4.70e-02$\pm$6.74e-02 & 143\% & $40.4\times$ \\
    SOAP+GradNorm  & 2.86e-05$\pm$2.19e-05 & 77\% & 3.16e-05$\pm$1.71e-05 & 3.41e-05$\pm$1.42e-05 & 3.23e-05$\pm$1.42e-05 & 3.22e-05$\pm$1.85e-05 & 2.69e-05$\pm$9.93e-06 & 37\% & $0.9\times$ \\
    \bottomrule
  \end{tabular}}
\end{table}

\subsection{Nernst--Planck--Poisson --- Final $L_2$}

\begin{table}[H]
  \centering
  \resizebox{\textwidth}{!}{%
  \begin{tabular}{@{}lccccccc@{}}
    \toprule
    Method & $\varepsilon\!=\!1$ & CV & $\varepsilon\!=\!0.5$ & $\varepsilon\!=\!0.2$ & $\varepsilon\!=\!0.1$ & CV & Deg \\
    \midrule
    Adam           & 1.39e-03$\pm$1.45e-03 & 105\% & 3.67e-03$\pm$2.25e-03 & 7.16e-02$\pm$9.20e-03 & FAIL ($0.27$) &  & FAIL \\
    SOAP           & 2.85e-05$\pm$7.97e-06 & 28\% & 2.04e-05$\pm$6.85e-06 & 1.57e-04$\pm$9.35e-05 & 3.37e-03$\pm$5.21e-04 & 15\% & $118.4\times$ \\
    Adam+GradNorm  & 1.96e-03$\pm$2.85e-03 & 145\% & 1.61e-03$\pm$2.11e-03 & FAIL ($0.12$) & FAIL ($0.61$) &  & FAIL \\
    SOAP+GradNorm  & 1.81e-05$\pm$1.56e-05 & 86\% & 3.48e-05$\pm$2.03e-05 & 3.27e-05$\pm$1.42e-05 & 4.20e-05$\pm$4.05e-05 & 96\% & $2.3\times$ \\
    \bottomrule
  \end{tabular}}
\end{table}

\subsection{2D NP+P+Stokes --- Training History}
\label{app:2d_training}

\Cref{fig:2d_adam_history} shows that Adam+GradNorm never brings $L_2$ below the $0.1$ failure threshold in the 50000 epochs; the persistent oscillation around a near-constant, quasi-neutral state indicates that longer training alone would not yield convergence. In \cref{fig:2d_soap_history}, SOAP+GradNorm steadily reduces $L_2$ after an initial transient, with errors concentrated near the Debye layer. \Cref{fig:2d_adam_sn_history,fig:2d_soap_sn_history} present the corresponding diagnostics at $\varepsilon = 0.01$ with the single-network architecture. Adam+GradNorm remains near the quasi-neutral state, whereas SOAP+GradNorm resolves the Debye layer.

\begin{figure}[H]
  \centering
  \includegraphics[width=\textwidth]{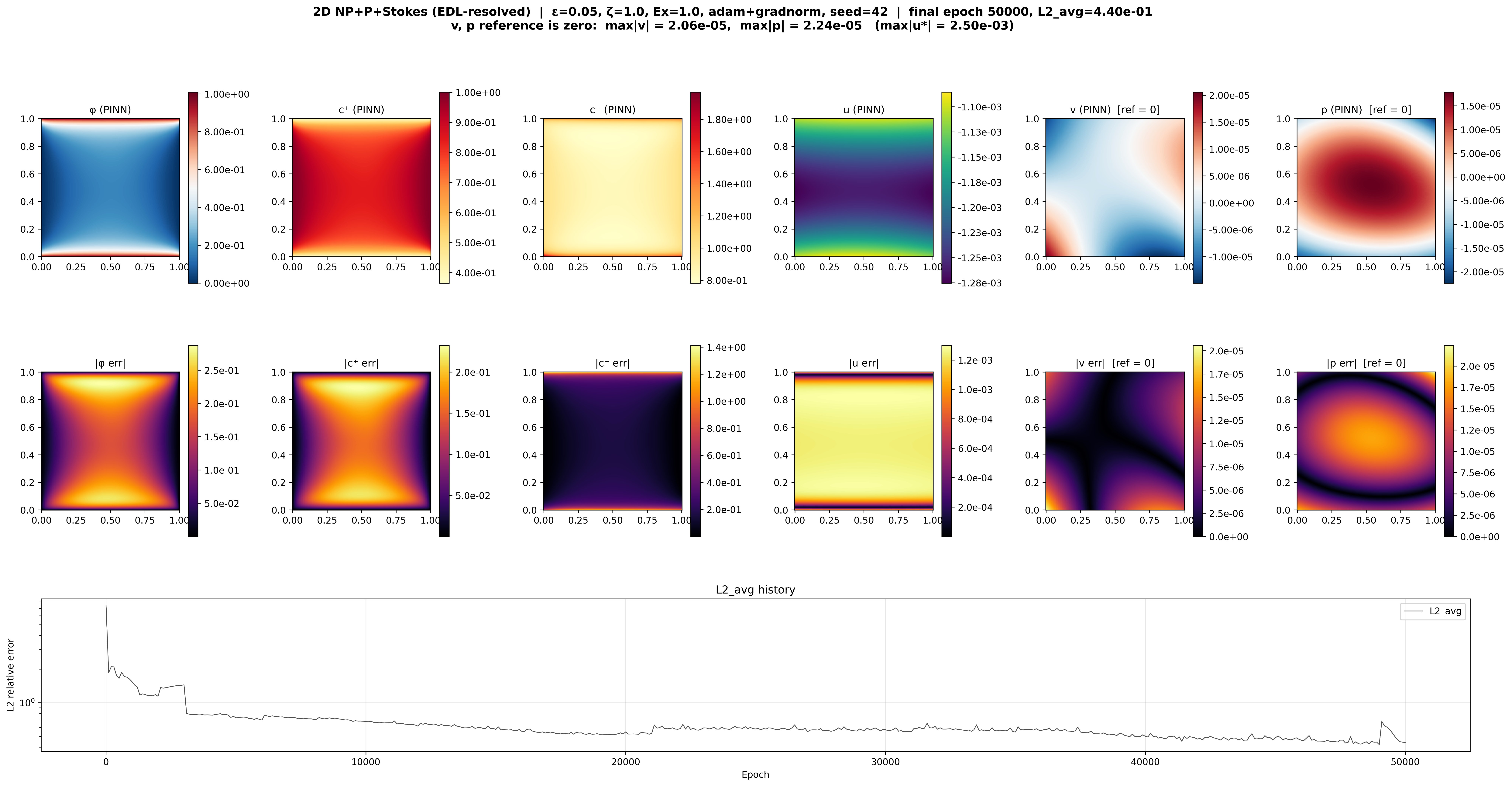}
  \caption{2D Nernst--Planck--Poisson (NP+P)+Stokes training diagnostics for Adam+GradNorm ($\varepsilon = 0.05$, $\zeta = 1.0$, single seed). Top rows: predicted fields at the final epoch and pointwise absolute errors; $v$ and $p$ are shown as deviations from a zero reference ($\max|\hat{v}| = 2.1 \times 10^{-5}$, $\max|\hat{p}| = 2.2 \times 10^{-5}$). Bottom: $L_2$ training history over 50000 epochs.}
  \label{fig:2d_adam_history}
\end{figure}

\begin{figure}[H]
  \centering
  \includegraphics[width=\textwidth]{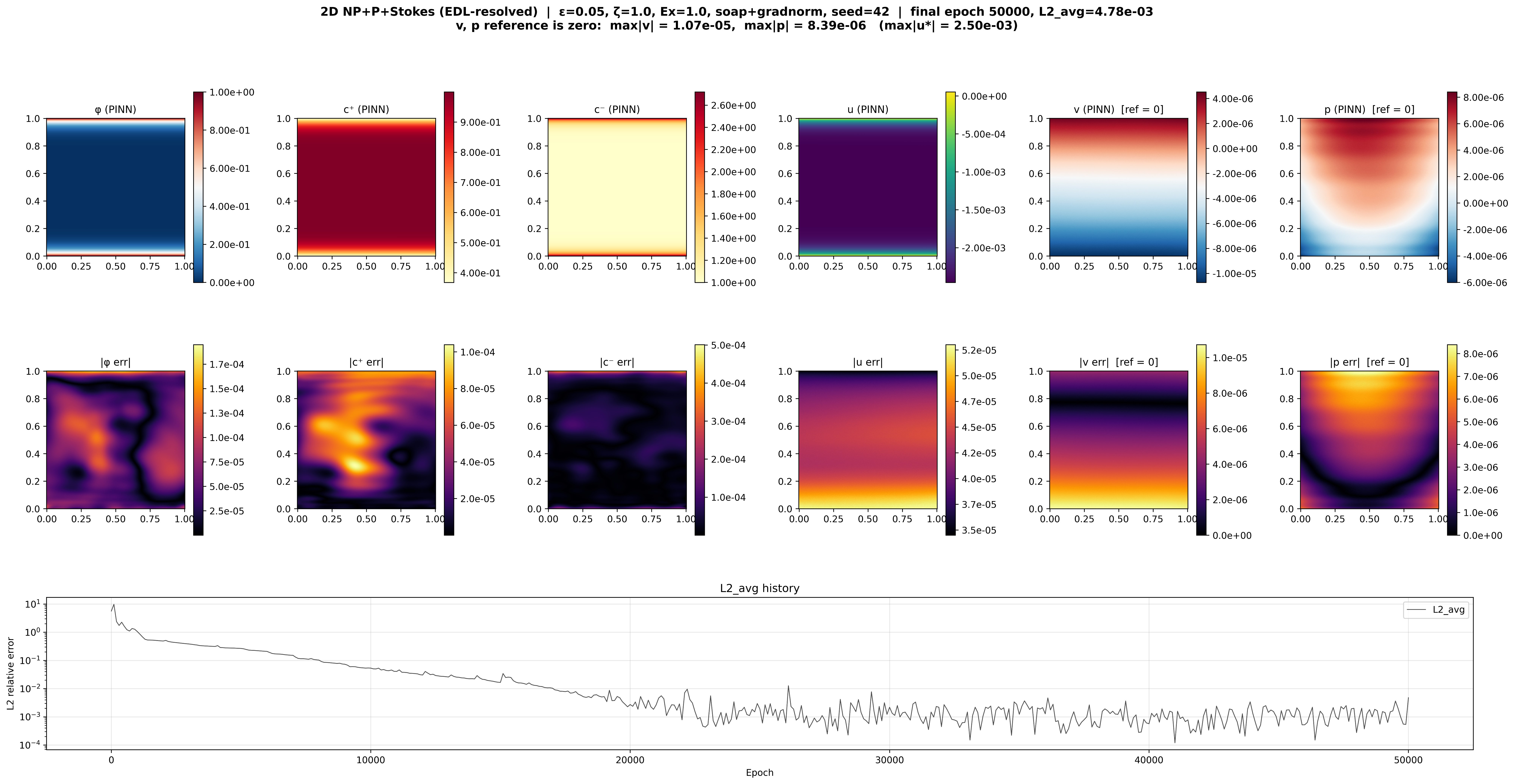}
  \caption{2D NP+P+Stokes training diagnostics for SOAP+GradNorm ($\varepsilon = 0.05$, $\zeta = 1.0$, single seed). Top rows: predicted fields at the final epoch and pointwise absolute errors; $v$ and $p$ are shown as deviations from a zero reference ($\max|\hat{v}| = 1.1 \times 10^{-5}$, $\max|\hat{p}| = 8.4 \times 10^{-6}$). Bottom: $L_2$ training history over 50000 epochs.}
  \label{fig:2d_soap_history}
\end{figure}

\begin{figure}[H]
  \centering
  \includegraphics[width=\textwidth]{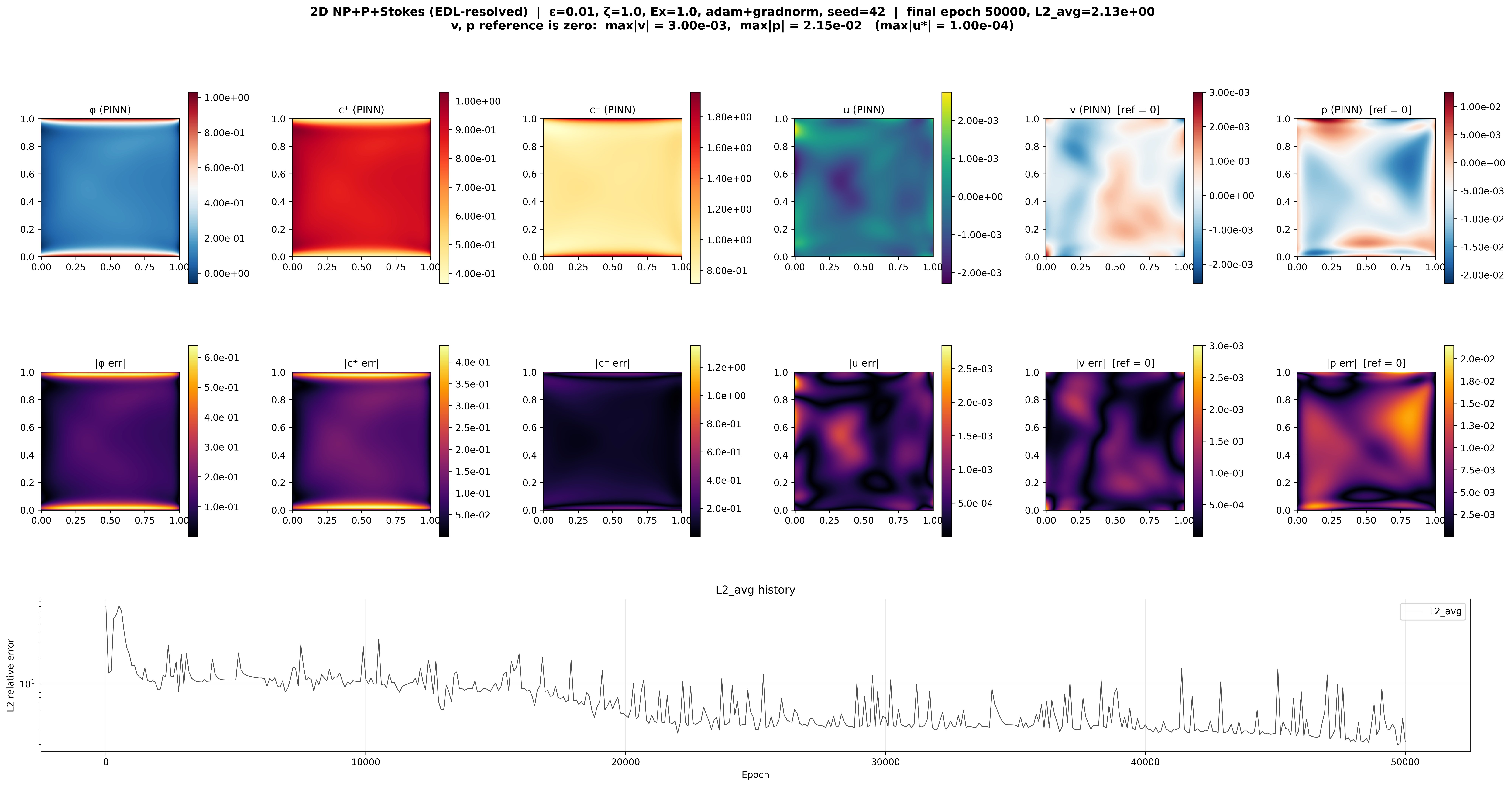}
  \caption{2D NP+P+Stokes training diagnostics for Adam+GradNorm with the single-network architecture ($\varepsilon = 0.01$, $\zeta = 1.0$, single seed). Top rows: predicted fields at the final epoch and pointwise absolute errors; $v$ and $p$ are shown as deviations from a zero reference ($\max|\hat{v}| = 3.0 \times 10^{-3}$, $\max|\hat{p}| = 2.1 \times 10^{-2}$). Bottom: $L_2$ training history over 50000 epochs.}
  \label{fig:2d_adam_sn_history}
\end{figure}

\begin{figure}[H]
  \centering
  \includegraphics[width=\textwidth]{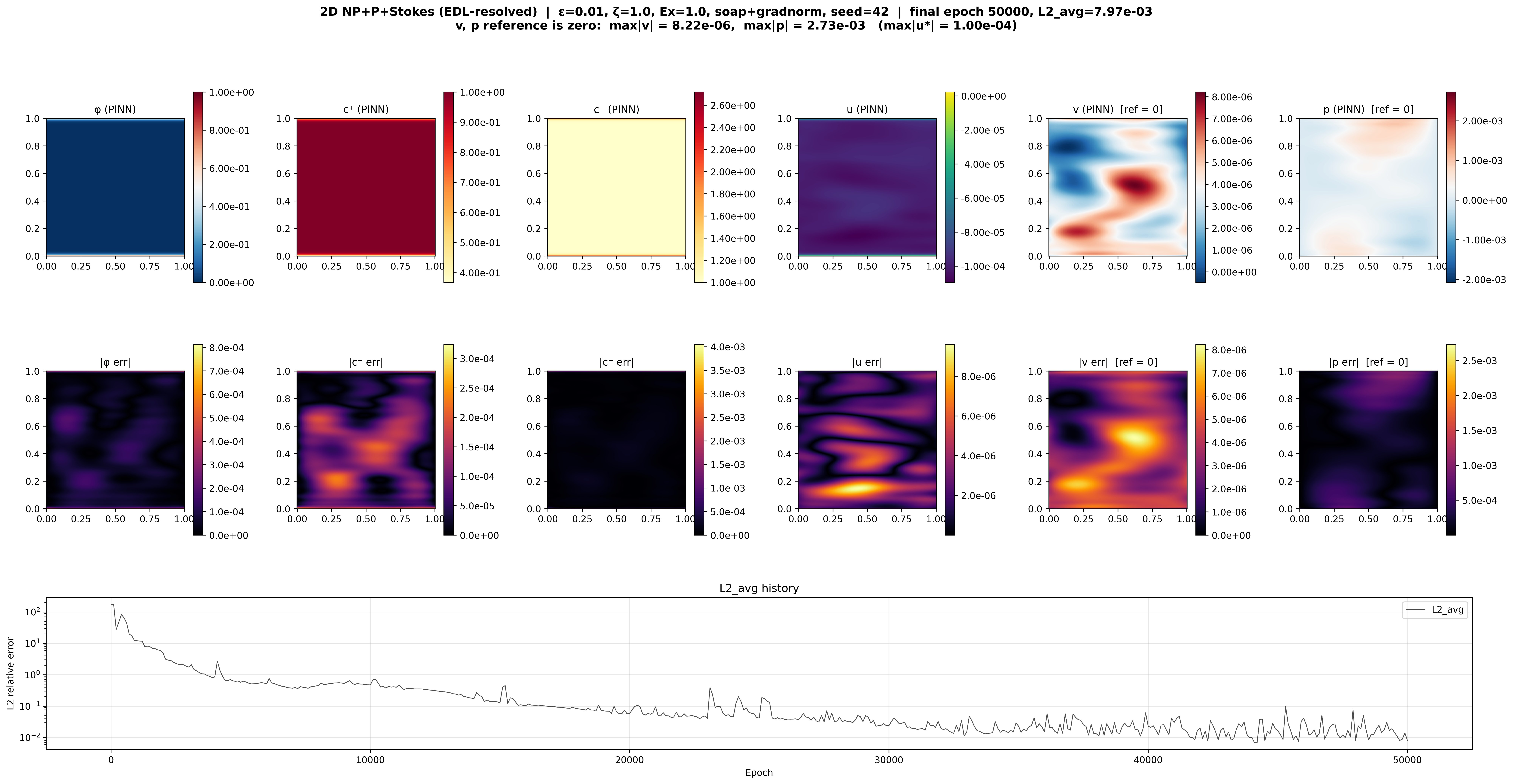}
  \caption{2D NP+P+Stokes training diagnostics for SOAP+GradNorm with the single-network architecture ($\varepsilon = 0.01$, $\zeta = 1.0$, single seed). Top rows: predicted fields at the final epoch and pointwise absolute errors; $v$ and $p$ are shown as deviations from a zero reference ($\max|\hat{v}| = 8.2 \times 10^{-6}$, $\max|\hat{p}| = 2.7 \times 10^{-3}$). Bottom: $L_2$ training history over 50000 epochs.}
  \label{fig:2d_soap_sn_history}
\end{figure}

\section{Implementation Details}
\label{app:implementation}

\begin{table}[H]
  \caption{Training configuration and hyperparameters.}
  \label{tab:implementation}
  \centering
  \small
  \begin{tabular}{@{}p{0.32\textwidth}p{0.62\textwidth}@{}}
    \toprule
    Parameter & Value \\
    \midrule
    \multicolumn{2}{@{}l}{\emph{Hardware and software}} \\
    Graphics processing unit (GPU) & NVIDIA GeForce RTX 4060 Ti (8\,GB), single GPU \\
    Environment & Windows; Python 3.11.14 (Anaconda); PyTorch 2.5.1 (CUDA 12.1, cuDNN 9.1.0) \\
    SOAP implementation & Official implementation of \citet{vyas2025soap} \\
    \midrule
    \multicolumn{2}{@{}l}{\emph{Architecture}} \\
    Networks (1D) & 4 hidden layers $\times$ 64 neurons per field (segregated, disjoint parameters) \\
    Networks (2D) & 5 hidden layers $\times$ 128 neurons per field \\
    Networks per system ($S$) & Thermoelasticity: 2; reaction--diffusion: 3; NP+P: 3; 2D NP+P+Stokes: 4 (segregated) or 1 (single-network) \\
    Network outputs & 1D: one scalar field per network. 2D segregated: three scalar networks plus one flow network with three outputs. 2D single-network: one network with six outputs \\
    Input dimension & One spatial coordinate (1D benchmarks); two spatial coordinates (2D benchmark) \\
    Activation / initialization & tanh / Xavier uniform \\
    Input normalization & $[-1, 1]$ \\
    \midrule
    \multicolumn{2}{@{}l}{\emph{Collocation points}} \\
    Thermoelasticity, reaction--diffusion & 200 interior (uniform) \\
    NP+P (1D) & 300 interior (half uniform, half in near-wall boundary layer) \\
    2D NP+P+Stokes & 3000 interior + 200 boundary per edge \\
    \midrule
    \multicolumn{2}{@{}l}{\emph{Optimization}} \\
    Adam & PyTorch defaults: $\beta = (0.9, 0.999)$, $\epsilon = 1\times10^{-8}$, weight decay $= 0$ \\
    SOAP & $\beta = (0.99, 0.999)$, preconditioner update frequency $=$ every 2 steps, weight decay $= 0$ \\
    Learning rate & $1 \times 10^{-3}$ (thermoelasticity, reaction--diffusion); $3 \times 10^{-4}$ (NP+P, 2D) \\
    Learning-rate schedule & Constant \\
    Batch size & Full batch (all collocation points per step) \\
    GradNorm & Update frequency $=$ 1000 steps, exponential moving average (EMA) momentum $= 0.9$ \\
    \midrule
    \multicolumn{2}{@{}l}{\emph{Training and evaluation}} \\
    Epochs & 30000 (1D); 50000 (2D) \\
    Seeds & Three per configuration \\
    Total runs & 210 (main study) + 12 ($\varepsilon = 0.01$ architecture comparison) $=$ 222 \\
    Exclusion criteria & Runs with final-epoch relative $L_2 > 0.1$ reported as FAIL; no runs otherwise excluded \\
    Metric & Final-epoch relative $L_2$, averaged over field variables and seeds \\
    Coupling degradation ratio & $\bar{L}_2(\text{strong}) / \bar{L}_2(\text{weak})$ \\
    \midrule
    \multicolumn{2}{@{}l}{\emph{Spectral measurements (\cref{sec:ntk_verification})}} \\
    Networks / collocation & 4 hidden layers $\times$ 64 neurons per field; 200 interior points \\
    Coupling grid & Thermoelasticity $\gamma \in \{0.1, 0.5, 1, 2, 5, 10, 25, 50, 100\}$; reaction--diffusion $k_{\mathrm{f}} \in \{0.5, 30\}$; NP+P $\varepsilon \in \{1.0, 0.1\}$ \\
    Rank tolerance & Singular vectors retained for $\sigma > 1\times10^{-14}\,\sigma_{\max}$; $D_{jj} < 1\times10^{-30}$ clamped to zero \\
    Initialization measurements & Three seeds per coupling value \\
    Training snapshots & Single seed, Adam+GradNorm, 30000 epochs; 16 snapshots at epochs 0, 2000, 4000, \ldots, 30000 \\
    Spectral run count & 39 initialization measurements + 6 snapshot runs; separate from the 222 accuracy runs \\
    \bottomrule
  \end{tabular}
\end{table}

\end{document}